\documentclass[sn-mathphys]{JORSC-jnl}

\jyear{2025}%

\theoremstyle{thmstyleone}%
\newtheorem{theorem}{Theorem}[section]
\newtheorem{proposition}[theorem]{Proposition}%
\newtheorem{lemma}[theorem]{Lemma}
\newtheorem{corollary}[theorem]{Corollary}
\theoremstyle{thmstyletwo}%
\newtheorem{example}{Example}[section]%
\newtheorem{remark}{Remark}[section]%

\theoremstyle{thmstylethree}%

\numberwithin{equation}{section}
\numberwithin{figure}{section}
\numberwithin{table}{section}

\usepackage{indentfirst}
\usepackage{epstopdf}
\usepackage{graphicx}
\usepackage{mathrsfs}
\usepackage{amsmath}
\usepackage{amssymb}
\usepackage{amsfonts}
\usepackage{wrapfig}
\newcommand\RR{\mathbb{R}}
\newcommand\Rd{\mathbb{R}^d}
\newcommand\NN{\mathbb{N}}
\newcommand\PP{\mathbb{P}}

\newcommand\vy{\boldsymbol{y}}
\newcommand\vc{\boldsymbol{c}}
\newcommand\va{\boldsymbol{a}}
\newcommand\vxi{\boldsymbol{\xi}}
\newcommand\ve{\boldsymbol{e}}
\newcommand\vb{\boldsymbol{b}}
\newcommand\vv{\boldsymbol{v}}
\newcommand\vt{\boldsymbol{t}}

\newcommand\vw{\boldsymbol{w}}
\newcommand\vsigma{\boldsymbol{\sigma}}
\newcommand\vA{\mathsf{A}}
\newcommand\vI{\mathsf{I}}
\newcommand\vW{\mathsf{W}}
\newcommand\Banach{\mathcal{B}}
\newcommand\Hilbert{\mathcal{H}}
\newcommand\RKBS{\mathcal{B}}
\newcommand\Space{\mathcal{X}}
\newcommand\Ball{B}
\newcommand\Sphere{S}
\newcommand\Value{Y}
\newcommand\Aset{\mathcal{A}}
\newcommand\Funset{\mathfrak{F}}
\newcommand\Sset{\mathcal{S}}
\newcommand\Iset{\mathcal{I}}
\newcommand\Eset{\mathcal{E}}

\newcommand\Cont{\mathrm{C}}
\newcommand\Leb{\mathrm{L}}
\newcommand\ud{\mathrm{d}}
\newcommand\Span{\mathrm{span}}
\newcommand\sign{\mathrm{sign}}
\newcommand\risk{R}
\newcommand\Measure{\mathcal{M}}
\newcommand\regfun{\Phi}
\newcommand{\abs}[1]{\left\lvert#1\right\rvert} 
\newcommand{\norm}[1]{\lVert#1\rVert}
\newcommand\AffLambda{\Psi}
\newcommand\afflambda{\psi}

\newcommand\Data{\mathcal{D}}
\newcommand\loss{\mathbb{L}}
\newcommand\closs{L}
\newcommand\Loss{\mathcal{L}}
\newcommand\Nueral{\mathcal{U}}
\newcommand\Bessel{\mathcal{K}}
\begin{document}
\renewcommand{\qedsymbol}{\Box}
\title[Regularized Learning in Banach Spaces]{Analysis of Regularized Learning in Banach Spaces for Linear-functional Data}


\author*{Qi Ye}

\affil{\orgdiv{School of Mathematical Sciences}, \orgname{South China Normal University}, \orgaddress{\city{Guangzhou} \postcode{510631}, \state{Guangdong}, \country{China}}}


\abstract{This article delves into the study of the theory of regularized learning in Banach spaces for linear-functional data. It encompasses discussions on representer theorems, pseudo-approximation theorems, and convergence theorems.
Regularized learning is designed to minimize regularized empirical risks over a Banach space.
The empirical risks are calculated by utilizing training data and multi-loss functions.
The input training data are composed of linear functionals in a predual space of the Banach space to capture discrete local information from multimodal data and multiscale models.
Through the regularized learning, approximations of the exact solution to an unidentified or uncertain original problem are globally achieved.
In the convergence theorems, the convergence of the approximate solutions to the exact solution is established through the utilization of the weak* topology of the Banach space.
The theorems of regularized learning are utilized in the interpretation of classical machine learning, such as support vector machines and artificial neural networks.}




\maketitle

\section{Introduction}\label{Intr}

In numerous practical problems, the main goal frequently involves determining an exact solution to minimize expected risks or errors over a Banach space, as exemplified in Optimization \eqref{eq:TRM}.
The exact solution will provide the optimal decision rule for ascertaining the most appropriate course of action.
Unfortunately, the expected risks are frequently unidentified or uncertain. Hence, it is difficult to determine the exact solution to the original problem directly.
In various mathematical models within the fields of physics and engineering, observed data are collected and measured from the original problem.
Moreover, the large-scale data are employed to approximate the exact solution.
Regularized learning provides an effective method for developing approximate solutions.
Regularized learning, as presented here, focuses on seeking an approximate solution to minimize regularized empirical risks that are the sum of empirical risks and regularization terms, as exemplified in Optimization \eqref{eq:RERM}.
The motivation for employing regularized learning is provided by the computation of empirical risks using finitely many training data and simplistic loss functions, as exemplified in Equation \eqref{eq:EmpRiskFun},
and the approximation of expected risks by empirical risks, as exemplified in Equation \eqref{eq:PointConv}.
Regularized learning plays a vital role in various disciplines, including statistical learning, regression analysis, approximation theory, inverse problems, and signal processing.
Regularized learning is applied in various fields, such as engineering, computer science, psychology, intelligent medicine, and economic decision making.

The theory of regularized learning has already achieved a success in reproducing kernel Hilbert spaces (RKHS), exemplified by support vector machines in \cite{SteinwartChristmann2008}.
The learning theory has recently been extended to reproducing kernel Banach spaces (RKBS) in \cite{XuYe2019,ZhangXuZhang2009}.
Regularized learning is commonly employed to analyze classical input data composed of regular vectors.
The reproduction of RKHS and RKBS guarantees that the classical input data can be equivalently transferred to the point evaluation functionals.
In our papers \cite{HuangLiuTanYe2020,Ye2019II,Ye2019I}, we explore a generalized concept of linear-functional data for capturing discrete local information of multimodal data and multiscale models.
As shown in Equation \eqref{eq:GenData}, the generalized input data are composed of bounded linear functionals.
In the same manner as \cite[Theorem 3.1]{HuangLiuTanYe2020}, we
generalize the representer theorem in a Banach space, which has a predual space.
This demonstrates that regularized learning can be a viable method for analyzing linear-functional data.
Our primary concept is that regularized learning is locally interpreted through linear-functional data and the exact solution is globally approximated through regularized learning.
This article delves into the entire theory of regularized learning through the utilization of the weak* topology of the Banach space.
It encompasses discussions on the representer theorems, pseudo-approximation theorems, and convergence theorems. The convergence theorems ensure that the approximate solutions converge to the exact solution, and the representer theorems or pseudo-approximation theorems ensure that the approximate solutions are computed either equivalently or approximately through finite-dimensional optimization.

\begin{figure}[H]
    \centering
    \includegraphics[width=0.98\linewidth]{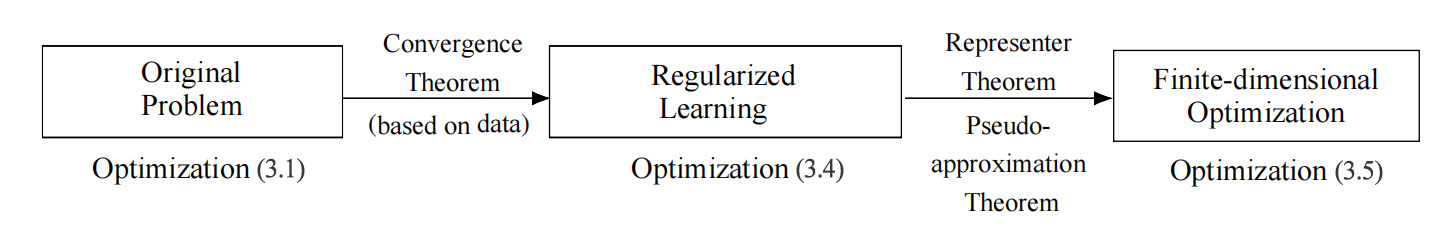}
\end{figure}

Lots of individuals demonstrate interest in the construction and implementation of regularized learning.
Currently, there is a limited number of papers \cite{CuckerSmale2002,SmaleZhou2004} that address the convergence of regularized learning for classical data. Nonetheless, there is still a dearth of research on the convergence analysis of regularized learning in Banach spaces for linear-functional data.
In Section \ref{sec:ConThm}, the convergence theorems concerning the approximate solutions in Banach spaces will be explored.
The predual space guarantees the existence of the weak* topology of the Banach space establishing connectivity between
the empirical risks and the regularization terms.
Even if the conditions of the uniform convergence or $\Gamma$-convergence of the empirical risks to the expected risks is not satisfied globally,
then the weak* compactness provides an alternative method to demonstrate the weak* convergence of the approximate solutions to the exact solution.
In contrast to the classical proof of convergence in learning theory, the assumption of independently and identically distributed data is not required in the proof of weak* convergence.

\begin{figure}[H]
    \centering
    \includegraphics[width=.80\linewidth]{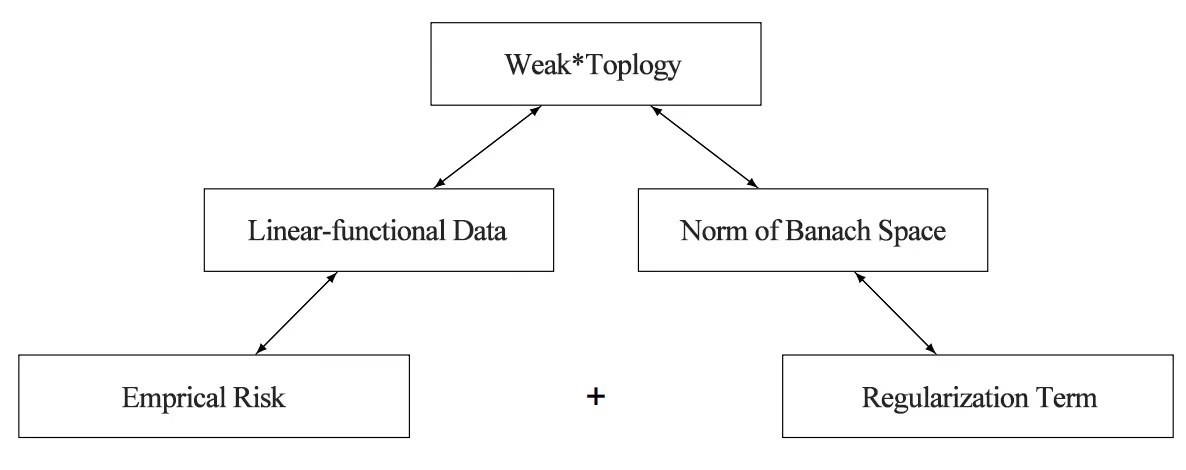}\par
    {Regularized Learning for Linear-functional Data}
\end{figure}

\begin{remark}
The formulas for regularized learning bear resemblance to classical constructions of inverse problems.
However, there are distinctions between the specifics of regularized learning and inverse problems.
For instance, regularized learning is typically computed through finitely many discrete data, while
inverse problems primarily concern continuous problems, such as integral equations.
We will finalize the proofs of the convergence theorems using the characteristics of the weak* topology.
Regularized learning can be considered an interdisciplinary field that encompasses approximation theory, optimization theory, and regularization theory.
\end{remark}

As shown in Examples \ref{exa:Gaussian} and \ref{exa:BinaryClass}, the classical input data can be equivalently transferred to the linear-functional data composed of point evaluation functions.
As shown in Example \ref{exa:Poisson}, linear-functional data have the capability to depict a wide range of data arising from partial differential equations.
In Theorem \ref{Thm:GenRepr} and Corollary \ref{Cor:GenReprExtPoint},
the techniques described in \cite{BoyerChambolleCastroETC2019,Unser2020} are employed to
improve the generalized representer theorems in \cite{HuangLiuTanYe2020} that is utilized for multi-loss functions.
This demonstrates that the approximate solutions can be equivalently computed through finite-dimensional optimization, such as support vector machines.
Under the condition of universal approximation, Theorem \ref{Thm:App} ensures that the approximate solutions can be approximately computed through finite-dimensional optimization, such as artificial neural networks.
Under Conditions (I) and (II) in Section \ref{sec:ConThm}, the convergence analysis of the approximate solutions is expounded in Theorems \ref{Thm:ConvThm} and \ref{Thm:ConvThm-subnet}.
Furthermore, Theorem \ref{Thm:ConvThm-lambda} presents a specialized approach for selecting the adaptive regularization parameters
for guaranteeing the convergence of the approximate solutions.
As shown in Example \ref{exa:finite-dim},
technique of regularization is necessary for ill-posed problems, even when Conditions (I) and (II) are satisfied.
We ultimately finalize the proofs of all the theorems in Section \ref{sec:PrfThm}.

\section{Notations and Preliminaries}\label{sec:NotPre}

In this section, we introduce several notations and results of functional analysis, along with some characteristics of Banach spaces that are commonly referenced in this article.
Here, we \emph{only} consider real-valued instances.

\subsection{Banach Spaces and Predual Spaces}\label{sec:BanachPredual}

Let $\Banach$ be a Banach space and let $\Banach^{\ast}$ be a dual space of $\Banach$.
Thus, $\Banach^{\ast}$ is a collection of all bounded linear functionals on $\Banach$.
We denote norms of $\Banach$ and $\Banach^{\ast}$ by $\norm{\cdot}$ and $\norm{\cdot}_{\ast}$, respectively, and a dual bilinear product of $\Banach$ and $\Banach^{\ast}$ by $\langle\cdot,\cdot\rangle$.
Let $\Ball_{\Banach}$ be a closed unit ball of $\Banach$ and let $\Sphere_{\Banach}$ be an unit sphere of $\Banach$.
Thus, $r\Ball_{\Banach}=\left\{f\in\Banach:\norm{f}\leq r\right\}$
and $r\Sphere_{\Banach}=\left\{f\in\Banach:\norm{f}=r\right\}$
for any $r>0$.
Let $\partial\norm{\cdot}$ be a subdifferential of $\norm{\cdot}$.
Thus, the characterizing subdifferential of norm in \cite[Corollary 2.4.16]{Zalinescu2002} shows that
\begin{equation}\label{eq:CharSubdiffNorm}
\partial\norm{\cdot}\left(f\right)=
\left\{
\begin{array}{cl}
\left\{\xi^{\ast}\in\Sphere_{\Banach^{\ast}}:\norm{f}=\langle f,\xi^{\ast}\rangle\right\},
&\text{if }f\in\Banach\setminus\left\{0\right\},\\
\Ball_{\Banach^{\ast}},&\text{otherwise}.
\end{array}
\right.
\end{equation}
Let $\Aset$ be any nonempty subset of $\Banach$.
We denote a dual cone of $\Aset$ by
$\Aset^{+}:=\left\{\xi^{\ast}\in\Banach^{\ast}:\langle f,\xi^{\ast}\rangle\geq0\text{ for all }f\in\Aset\right\}$,
and an orthogonal space of $\Aset$ by
$\Aset^{\bot}:=\left\{\xi^{\ast}\in\Banach^{\ast}:\langle f,\xi^{\ast}\rangle=0\text{ for all }f\in\Aset\right\}$.
Specifically, if $\Aset$ is a closed subspace of $\Banach$,
then $\Aset^{+}=\Aset^{\bot}$ and $\Aset^{\bot\bot}=\Aset$.
Let $\chi_{\Aset}$ be an indicator function of $\Aset$, that is,
$\chi_{\Aset}(f):=0$ if $f\in\Aset$ and $\chi_{\Aset}(f):=\infty$ otherwise.
As stated in \cite[Section 2.3]{Zalinescu2002}, if $\Aset$ is a closed affine space of $X$ and $f\in\Aset$, then
\begin{equation}\label{eq:NormalConeAffine}
-\partial\chi_{\Aset}\left(f\right)
=\left(\text{cone}\left(\Aset-f\right)\right)^{+}
=\left(\Aset-f\right)^{\bot}.
\end{equation}
Moreover, the dual bilinear product $\langle\cdot,\cdot\rangle$ is extended to
\[
\langle f,\vxi^{\ast}\rangle:=\left(\langle f,\xi^{\ast}_{1}\rangle,\langle f,\xi^{\ast}_{2}\rangle,\cdots,\langle f,\xi^{\ast}_{N}\rangle\right),
\]
where $f\in\Banach$ and $\vxi^{\ast}=\left(\xi_{1}^{\ast},\xi_{2}^{\ast},\cdots,\xi_{N}^{\ast}\right)$
composed of $\xi_{1}^{\ast},\xi_{2}^{\ast},\ldots,\xi_{N}^{\ast}\in\Banach^{\ast}$.
We say that $ker\left(\vxi^{\ast}\right)$ is a kernel of $\vxi^{\ast}$, that is,
\[
ker\left(\vxi^{\ast}\right):=\left\{f\in\Banach:\langle f,\vxi^{\ast}\rangle=\boldsymbol{0}\right\}.
\]
Thus, $ker\left(\vxi^{\ast}\right)$ is a closed linear subspace of $\Banach$ and
\begin{equation}\label{eq:KerLinSpan}
ker\left(\vxi^{\ast}\right)^{\bot}=
\Span\left\{\xi^{\ast}_{1},\xi^{\ast}_{2},\ldots,\xi^{\ast}_{N}\right\}.
\end{equation}

As stated in \cite[Definition 2.2.27]{DalesDashiellLauStrauss2016},
we say that $\Banach_{\ast}$ is a \emph{predual space} of $\Banach$ if $\Banach_{\ast}$ is a complete subspace of $\Banach^{\ast}$ under the induced norm
and the dual space of $\Banach_{\ast}$ is isometrically isomorphic to $\Banach$, that is, $\left(\Banach_{\ast}\right)^{\ast}\cong\Banach$.
For example, if $\Space$ is a Banach space and $\Banach=\Space^{\ast}$,
then $\Banach$ has a predual space $\Banach_{\ast}\cong\Space$.
Specifically, if $\Banach$ is reflexive, then $\Banach_{\ast}=\Banach^{\ast}$.
The predual space $\Banach_{\ast}$ guarantees that
the \emph{weak* topology} of $\Banach$ is introduced by the topologizing family $\Banach_{\ast}$ and
all weak* characteristics hold true on $\Banach$.
We say that a net $\left(f_{\alpha}\right)\subseteq\Banach$ weakly* converges to an element $f\in\Banach$ if and only if
$\lim_{\alpha}\langle f_{\alpha},\xi^{\ast}\rangle=\langle f,\xi^{\ast}\rangle$ for all $\xi^{\ast}\in\Banach_{\ast}$.
We will study the weak* convergence in $\Banach$ by the weak* topology $\sigma\left(\Banach,\Banach_{\ast}\right)$.
For one example, the Riesz–Markov theorem ensures that
the bounded total variation space $\Measure(\Rd)$ composed of regular countably additive Borel measures on $\Rd$ has a predual space, 
as exemplified in \cite[Generalized Total Variation]{UnserFageotWard2017}.
For another example,
if $\Omega$ is a compact domain,
then \cite[Theorem 6.4.1]{DalesDashiellLauStrauss2016} ensures that
$\Cont(\Omega)$ has a predual space if and only if $\Omega$ is hyper-Stonean.
In this article, we assume that $\Banach$ \emph{always} has a predual space $\Banach_{\ast}$ for convenience.

The Banach-Alaoglu theorem ensures that $r\Ball_{\Banach}$ is weakly* compact.
Thus, $r\Ball_{\Banach}$ with respect to the relative weak* topology is a compact Hausdorff space.
In this article, we will consider a specific continuity on $r\Ball_{\Banach}$.
Let a function $R\in\RR^{\Banach}$.
The $R$ is \emph{weakly* lower semi-continuous} or \emph{weakly* continuous} on $r\Ball_{\Banach}$ if and only if
\[
R\left(f_0\right)\leq\liminf_{\alpha}R\left(f_{\alpha}\right)
\text{ or }R\left(f_0\right)=\lim_{\alpha}R\left(f_{\alpha}\right),
\]
for any weakly* convergent net $\left(f_{\alpha}\right)\subseteq r\Ball_{\Banach}$ to an element $f_0\in r\Ball_{\Banach}$.
Obviously, if $R$ is weakly* lower semi-continuous or weakly* continuous, then
$R$ is weakly* lower semi-continuous or weakly* continuous on $r\Ball_{\Banach}$ for all $r>0$.
For example, since the norm function $\norm{\cdot}$ is weakly* lower semi-continuous and the linear functional $\xi^{\ast}\in\Banach_{\ast}$ is weakly* continuous,
$\norm{\cdot}$ is weakly* lower semi-continuous on $r\Ball_{\Banach}$ for all $r>0$
and $\xi^{\ast}$ is weakly* continuous on $r\Ball_{\Banach}$ for all $r>0$.
Let $\Cont\left(r\Ball_{\Banach}\right)$ be a collection of all continuous functions on
$r\Ball_{\Banach}$ with respect to the relative weak* topology.
Thus, $R$ is weakly* continuous on $r\Ball_{\Banach}$ if and only if $R\in\Cont\left(r\Ball_{\Banach}\right)$.
To simplify the notations, the restriction $R\mid_{r\Ball_{\Banach}}$ is rewritten as $R$.
Let a set $\Funset\subseteq\RR^{\Banach}$. The $\Funset$ is \emph{weakly* equicontinuous} on $r\Ball_{\Banach}$
if and only if
\[
\lim_{\alpha}\sup_{R\in\Funset}\abs{R\left(f_0\right)-R\left(f_{\alpha}\right)}=0,
\]
for any weakly* convergent net $\left(f_{\alpha}\right)\subseteq r\Ball_{\Banach}$ to an element $f_0\in r\Ball_{\Banach}$.
Now, we study the characteristics of $\Funset$.
\begin{lemma}\label{Lm:CompactEquicontinuous}
If $\Funset\subseteq\Banach_{\ast}$, then the following are equivalent.
\begin{itemize}
\item[(a)] The $\Funset$ is weakly* equicontinuous on $\Ball_{\Banach}$.
\item[(b)] The $\Funset$ is weakly* equicontinuous on $r\Ball_{\Banach}$ for all $r>0$.
\item[(c)] The $\Funset$ is relatively compact in $\Banach_{\ast}$.
\end{itemize}
\end{lemma}

Obviously, Lemma \ref{Lm:CompactEquicontinuous} (b) is also equivalent that
$\Funset$ is weakly* equicontinuous on all bounded subset of $\Banach$.

\begin{remark}
In this article, the notations of machine (statistical) learning in \cite{HuangLiuTanYe2020,XuYe2019} will be integrated with the notations of topological vector spaces in \cite{Megginson1998,Zalinescu2002}.
We primarily utilize the symbol system outlined in \cite{HuangLiuTanYe2020}. Specifically,
we denote input and output by $\left(\xi^{\ast},y\right)$.
As stated in \cite[Definitions 2.1.8 and 2.1.26]{Megginson1998}, the theory of nets and subnets is utilized to study the weak* characteristics of Banach spaces.
\end{remark}

\subsection{Motivations of Linear-functional Data}\label{sec:MotLinData}

For classical machine learning of regression and classification,
we will approximate an exact solution using classical input data
\[
x_1,x_2,\ldots,x_N\in\Omega,
\]
where $\Omega$ is a subset of $\Rd$.
If the machine learning problem issues in $\Hilbert_K(\Omega)$, then the representer theorem in RKHS in \cite{SteinwartChristmann2008} ensures that the approximate solution can be expressed as a linear combination of a kernel basis
\[
K\left(\cdot,x_1\right),K\left(\cdot,x_2\right),\ldots,K\left(\cdot,x_N\right)\in\Hilbert_K(\Omega),
\]
where $\Hilbert_K(\Omega)$ is a RKHS and $K:\Omega\times\Omega\to\RR$ is its reproducing kernel.
The proof of the representer theorem in RKHS relies on the reproduction of RKHS which demonstrates that
$x_1,x_2,\ldots,x_N$ can be equivalently transferred to
\[
\delta_{x_1},\delta_{x_2},\ldots,\delta_{x_N}\in\left(\Hilbert_K(\Omega)\right)_{\ast}\cong\Hilbert_K(\Omega),
\]
where $\delta_{x}$ is a \emph{point evaluation functional} for a $x\in\Omega$, that is,
$\langle f,\delta_{x}\rangle=f(x)$ for any $f\in\RR^{\Omega}$.
Recently, the representer theorem in RKBS in \cite{XuYe2019} ensures that the approximate solution in $\RKBS_K^p(\Omega)$ also relies on the reproduction of RKBS which demonstrates that $x_1,x_2,\ldots,x_N$ can be equivalently transferred to
\[
\delta_{x_1},\delta_{x_2},\ldots,\delta_{x_N}\in\left(\RKBS_K^p(\Omega)\right)_{\ast}\cong\RKBS_K^q(\Omega),
\]
where $\RKBS_K^p(\Omega)$ is a $p$-norm RKBS for $1\leq p<\infty$ and $p,q$ is a pair of conjugate exponents. Specifically,
$\RKBS_K^2(\Omega)=\Hilbert_K(\Omega)$ and
$\RKBS_K^1(\Omega)$ is nonreflexive, nonstrictly convex, and nonsmooth.
Therefore, this concept prompts an exploration of machine learning problems through \emph{linear-functional data}
\[
\xi_{1}^{\ast},
\xi_{2}^{\ast},\ldots,\xi_{N}^{\ast}\in\Banach_{\ast}.
\]
This shows that linear-functional data can be a discretization of integral equations and differential equations, for example
$\langle f,\xi^{\ast}\rangle:=\int_{\Omega}f(x)\mu(\ud x)$ and $\langle f,\xi^{\ast}\rangle:=\partial f(x)$ as shown in Examples \ref{exa:BinaryClass} and \ref{exa:Poisson}, respectively.
Thus, linear-functional can be viewed as not only a vector or a point but also an operator or a distribution.
Roughly speaking, the predual space can be viewed as an extension of a classical domain that encompasses elements of point evaluation, differentiation, and integration.
This demonstrates the potential of utilizing linear-functional data to capture local information from various models for the development of global approximate solutions through regularized learning.
More detailed information regarding linear-functional data and associated concepts can be found in \cite{HuangLiuTanYe2020,Unser2020,WangXu2019,Ye2019II,Ye2019I}.
For instance, the linear-functional data are defined in predual spaces to introduce the representer theorems in Banach spaces in \cite{HuangLiuTanYe2020,WangXu2019}, while the linear-functional data are defined in Banach spaces to introduce the representer theorems in dual spaces in \cite{Unser2020}.

Let $\Value$ be a collection of all output elements dependent on various machine learning problems, such as $\Value:=\RR$ for regression and $\Value:=\{\pm1\}$ for classification.
We focus on deterministic data. Actually, the output data can be extended to stochastic data. More precisely, $\Value$ is a collection of random variables with discrete or continuous distributions, as shown in Example \ref{exa:BinaryClass}.

In this article, we consider \emph{training data}
\begin{equation}\label{eq:GenData}
\left(\xi_{n1}^{\ast},y_{n1}\right),
\left(\xi_{n2}^{\ast},y_{n2}\right),\ldots,\left(\xi_{nN_n}^{\ast},y_{nN_n}\right)
\in\Banach_{\ast}
\times\Value,
\quad\text{for all }n\in\NN,
\end{equation}
where $N_n\in\NN$.
Usually $N_n\to\infty$ when $n\to\infty$.
For any $n$th approximation step, we collect $N_n$ observations, and they may be partial duplicate from different steps.
Let
\[
\vxi^{\ast}_{n}:=\left(\xi_{n1}^{\ast},\xi_{n2}^{\ast},\cdots,\xi_{nN_n}^{\ast}\right),
\quad
\vy_{n}:=\left(y_{n1},y_{n2},\cdots,y_{nN_n}\right),
\quad\text{for all }n\in\NN.
\]
Thus, Equation \eqref{eq:GenData} is rewritten as
\begin{equation}\label{eq:GenData-vec}
\left(\vxi^{\ast}_{n},\vy_{n}\right)\in\Banach_{\ast}^{N_n}\times\Value^{N_n},
\quad\text{for all }n\in\NN,
\end{equation}
where
$\Banach_{\ast}^{N_n}:=\prod_{k=1}^{N_n}\Banach_{\ast}$
and 
$\Value^{N_n}:=\prod_{k=1}^{N_n}\Value$.
We denote a collection of all training data by
\[
\Data:=\left\{\left(\vxi^{\ast}_{n},\vy_{n}\right):n\in\NN\right\}.
\]

Next, we denote a collection of all input data by
\[
\Funset_{\Data}:=\left\{\xi_{nk_n}^{\ast}:k_n\in\NN_{N_n},n\in\NN\right\},
\]
where $\NN_N:=\left\{1,2,\ldots,N\right\}$ for any $N\in\NN$.
Usually, $\Funset_{\Data}$ is an infinite countable subset of $\Banach_{\ast}$.
Based on the construction of $\Funset_{\Data}$, we will consider a specific condition of $\Funset_{\Data}$ dependent on the norm of $\Banach$ to verify the convergence of the approximate solutions to the exact solutions, such as the equivalent conditions of
the weak* equicontinuity of $\Funset_{\Data}$ on $\Ball_{\Banach}$ and the relative compactness of $\Funset_{\Data}$ in $\Banach_{\ast}$ shown in Lemma \ref{Lm:CompactEquicontinuous}.
We will demonstrate that the classical input data composed of point evaluation functionals satisfies the specific condition.
Let $\Cont^{0,\vartheta}(\Omega)$ be a H\"{o}lder continuous space for an exponent $\vartheta>0$. According to \cite[Chapter 4]{XuYe2019}, $\RKBS_K^p(\Omega)$ is embedded into $\Cont^{0,\vartheta}(\Omega)$ for various reproducing kernels.
\begin{lemma}\label{Lm:weak-equicont}
Suppose that $\Omega$ is a compact set and $\delta_{x}\in\Banach_{\ast}$ for all $x\in\Omega$.
If $\Banach$ is embedded into $\Cont^{0,\vartheta}(\Omega)$,
then $\left\{\delta_{x_n}:n\in\NN\right\}$ is relatively compact in $\Banach_{\ast}$
for any $\left\{x_n:n\in\NN\right\}\subseteq\Omega$.
\end{lemma}

\begin{remark}\label{Rm:Equicontin-data}
It can be readily verified that $\Funset_{\Data}$ is equicontinuous, when $\Funset_{\Data}$ is uniformly bounded.
However, equicontinuity is not helpful for verifying weak* convergence.
Weak* equicontinuity typically serves as a crucial condition in proving weak* convergence.
Unfortunately, the condition of weak* equicontinuity is excessively strict to be validated using real data because a weakly* convergent subnet
may be divergent or unbounded.
We are considering a less stringent condition that can be easily verified in numerous machine learning problems.
For classical input data, Lemma \ref{Lm:weak-equicont} ensures that $\Funset_{\Data}$ satisfies relative compactness in $\Banach_{\ast}$ which is equivalent to weak* equicontinuity on $\Ball_{\Banach}$.
This illustrates that the relative compactness of $\Funset_{\Data}$ in $\Banach_{\ast}$ is a fundamental condition.
Moreover, if $\Funset_{\Data}$ is weakly* equicontinuous or bounded weakly* equicontinuous, then
$\Funset_{\Data}$ is weakly* equicontinuous on $\Ball_{\Banach}$.
If $\Funset_{\Data}$ is weakly* equicontinuous on $\Ball_{\Banach}$,
then $\Funset_{\Data}$ is equicontinuous.
The weak* equicontinuity of $\Funset_{\Data}$ on $\Ball_{\Banach}$ can be viewed as an intermediate condition.
The intermediate condition of linear-functional data also implies the specific weak* equicontinuity of empirical risk functions, exemplified by Condition (II$'$) in Section \ref{sec:ConThm} for proving the convergence theorems.
\end{remark}

\subsection{Extensions of Loss Functions}\label{sec:LossERisk}

By employing classical machine learning, we typically utilize a single loss function
$\closs:\Banach_{\ast}\times\Value\times\RR\to[0,\infty)$ to compute
empirical risks, such as
\begin{equation}\label{eq:EmpRisk-1}
\text{risk}=
\frac{1}{N_n}\sum_{k=1}^{N_n}\closs\left(\xi_{nk}^{\ast},y_{nk},\langle f,\xi_{nk}^{\ast}\rangle\right),
\quad\text{for }f\in\Banach.
\end{equation}
We aim to integrate diverse data from multiple models in order to formulate novel machine learning algorithms.
For sophisticated models, empirical risks can be computed through the utilization of diverse loss functions, such as
\begin{equation}\label{eq:EmpRisk-2}
\text{risk}=
\frac{1}{M_n}\sum_{j=1}^{M_n}
\frac{1}{N_{n}^{j}}\sum_{k=1}^{N_{n}^{j}}\closs^{j}_n\big(\xi_{nk}^{j\ast},y_{nk}^{j},\langle f,\xi_{nk}^{j\ast}\rangle\big),
\quad\text{for }f\in\Banach,
\end{equation}
where $\closs^j_n$ is a loss function related to  $\big(\xi_{n1}^{j\ast},y_{n1}^{j}\big),\big(\xi_{n2}^{j\ast},y_{n2}^{j}\big),\ldots,\big(\xi_{nN_n^j}^{j\ast},y_{nN_n^j}^{j}\big)$
for $j=1,2,\ldots,M_n$.

Now, we extend the concept of loss functions to cover the standard formulas in Equation \eqref{eq:EmpRisk-1}, the complicated formulas in Equation \eqref{eq:EmpRisk-2}, and other general formulas.
For any $n$th approximation step, we denote a \emph{multi-loss function} by
\begin{equation}\label{eq:lossfun}
\loss_n:\Banach_{\ast}^{N_n}\times\Value^{N_n}\times\RR^{N_n}
\to[0,\infty),
\end{equation}
where $N_n$ is the number of the related data $\left(\vxi^{\ast}_{n},\vy_{n}\right)$.
For example of classical machine learning, we have
\begin{equation}\label{eq:lossfun-exa}
\loss_n\left(\vxi^{\ast},\vy,\vt\right)=
\frac{1}{N_n}\sum_{k=1}^{N_n}\closs\left(\xi_{k}^{\ast},y_{k},t_{k}\right),\quad
\text{for }\vxi^{\ast}\in\Banach_{\ast}^{N_n},~\vy\in\Value^{N_n},~\vt\in\RR^{N_n},
\end{equation}
where $\vxi^{\ast}=\left(\xi_{1}^{\ast},\xi_{2}^{\ast},\cdots,\xi_{N_n}^{\ast}\right)$,
$\vy=\left(y_{1},y_{2},\cdots,y_{N_n}\right)$, and $\vt=\left(t_{1},t_{2},\cdots,t_{N_n}\right)$.
We say that $\loss_n$ is a \emph{lower semi-continuous, continuous, or convex multi-loss function}
if $\loss_n\left(\vxi^{\ast},\vy,\cdot\right)$ is lower semi-continuous, continuous, or convex for all $\vxi^{\ast}\in\Banach_{\ast}^{N_n}$ and all $\vy\in\Value^{N_n}$.
We say that $\loss_n$ is a \emph{local Lipschitz continuous multi-loss function}
if for any $\theta>0$, there exists a $C_n^{\theta}>0$ such that
\[
\sup_{\vxi^{\ast}\in\Banach_{\ast}^{N_n},\vy\in\Value^{N_n}}
\abs{\loss_n\left(\vxi^{\ast},\vy,\vt_1\right)-\loss_n\left(\vxi^{\ast},\vy,\vt_2\right)}
\leq C_n^{\theta}\norm{\vt_1-\vt_2}_{\infty},
\]
for all $\vt_1,\vt_2\in[-\theta,\theta]^{N_n}$.
We denote a collection of all multi-loss functions
by
\[
\Loss:=\left\{\loss_n:n\in\NN\right\}.
\]
We say that $\Loss$ is \emph{uniformly local Lipschitz continuous} if $\sup_{n\in\NN}C_n^{\theta}<\infty$ for all $\theta$.
The local Lipschitz continuity is a prevalent condition of loss functions that is similar to \cite[Definition 2.18]{SteinwartChristmann2008} for support vector machines.
\begin{lemma}\label{Lm:loss}
If $\loss_n$ has the formula as in Equation \eqref{eq:lossfun-exa}, then the following hold.
\begin{itemize}
\item[(a)] If $\closs$ is a lower semi-continuous, continuous, convex, or local Lipschitz continuous loss function,
then $\loss_n$ is a lower semi-continuous, continuous, convex, or local Lipschitz continuous multi-loss function.
\item[(b)] If $\closs$ is a local Lipschitz continuous loss function, then $\Loss$ is uniformly local Lipschitz continuous.
\end{itemize}
\end{lemma}

\section{Regularized Learning in Banach Spaces}\label{sec:RegLearnApp}

In numerous machine learning problems, the primary objective is to identify an \emph{exact solution} $f^0$ to minimize the expected risks (errors) over a Banach space $\Banach$, that is,
\begin{equation}\label{eq:TRM}
\risk\big(f^0\big)=
\inf_{f\in\Banach}\risk(f),
\end{equation}
where $\risk:\Banach\to[0,\infty)$ is an \emph{expected risk function}.
Let $\Sset^{0}\left(\Banach\right)$ be the collection of all minimizers of Optimization \eqref{eq:TRM}.
Thus, $f^0\in\Sset^{0}\left(\Banach\right)$.
Unfortunately, $\risk$ is frequently unidentified or uncertain, which complicates the resolution of Optimization \eqref{eq:TRM}.
To address the abstract problem, we approximate the exact solution using training data.

\begin{remark}
In practical applications, any minimizer of Optimization \eqref{eq:TRM} is a feasible solution.
Hence, the assumption of the singleton of $\Sset^{0}\left(\Banach\right)$ is unnecessary.
If $\Banach$ is a dense subspace of a topological vector space $\Space$ and $\risk$ is continuously extended on $\Space$, then Optimization \eqref{eq:TRM} can be viewed as an equivalent problem of the minimization of $\risk$ over
$\Space$. Under some simple conditions of $\risk$, the extension of Optimization \eqref{eq:TRM} exists a minimizer.
Thus, the existence of $f^0$ is assumed in the following discussions.
In this article, our focus is solely on the minimization of $\risk$ over $\Banach$ that has a predual space $\Banach_{\ast}$ ensuring the well-defined nature of the weak* topology of $\Banach$.
\end{remark}

\subsection{Regularized Learning for Linear-functional Data}\label{sec:RegLearnLinData}

For any $n$th approximation step, we utilize the training data $\left(\vxi^{\ast}_{n},\vy_{n}\right)$ in Equation \eqref{eq:GenData-vec} and the multi-loss function $\loss_n$ in Equation \eqref{eq:lossfun} to construct an \emph{empirical risk function} $\risk_n$, that is,
\begin{equation}\label{eq:EmpRiskFun}
\risk_{n}(f):=\loss_n\left(\vxi^{\ast}_{n},\vy_{n},\langle f,\vxi_n^{\ast}\rangle\right),
\quad\text{for }f\in\Banach.
\end{equation}
Thus, the empirical risks are computable through the utilization of training data and multi-loss functions.
The output $\vy_{n}$ can be viewed as the values of the exact solution at the input $\vxi_{n}^{\ast}$ without or with noises, for example of noiseless data,
$\vy_{n}=\langle f^0,\vxi_{n}^{\ast}\rangle$ in regression and
$\vy_{n}=\sign\left(\langle f^0,\vxi_{n}^{\ast}\rangle\right)$ in classification.
Roughly speaking, the approximate solution driven from the training data can be viewed as an equivalent element of the normal vector of a decision rule.
Furthermore, the multi-loss functions can be nonsmooth and nonconvex.

In the same manner as classical learning theory, empirical risks are employed to approximate expected risks, and empirical risks are considered as explicitly computable discretizations of expected risks.
Therefore, $\Data$ and $\Loss$ are chosen such that
$\risk_n$ converges pointwise to $R$ when $n\to\infty$, that is,
\begin{equation}\label{eq:PointConv}
\risk(f)=
\lim_{n\to\infty}\risk_n\left(f\right),
\quad\text{for all }f\in\Banach.
\end{equation}
The pointwise convergence represents a weak condition that is applicable in practical problems, even in case where $\risk_n$ is employed to approximate a simplified form of $\risk$.
In the field of learning theory, the viability of regularized learning derives from an approximation law using a large amount of data.
Therefore, according to Equation \eqref{eq:PointConv}, the approximate solutions will be determined by employing the empirical risk functions.
In order to verify generalization and prevent overfitting,
we determine an \emph{approximate solution} $f_{n}^{\lambda}$ to minimize the regularized empirical risks over $\Banach$, that is,
\begin{equation}\label{eq:RERM}
\risk_{n}\big(f_{n}^{\lambda}\big)
+\lambda\regfun\big(\norm{f_{n}^{\lambda}}\big)
=
\inf_{f\in\Banach}
\risk_{n}(f)
+\lambda\regfun\left(\norm{f}\right),
\end{equation}
where $\lambda>0$ is a regularization parameter and
$\regfun:[0,\infty)\to[0,\infty)$ is a continuous strictly increasing function such that $\regfun\left(r\right)\to\infty$ when $r\to\infty$.
Let $\Sset_{n}^{\lambda}\left(\Banach\right)$ be a collection of all minimizers of Optimization \eqref{eq:RERM}.
Since $\regfun$ is fixed and unfocused, we do not index $\Sset_{n}^{\lambda}\left(\Banach\right)$ for $\regfun$.
Thus, $f_{n}^{\lambda}\in\Sset_{n}^{\lambda}\left(\Banach\right)$.
A specific example of regularized learning is a classical binary classifier for hinge loss, that is
\[
\inf_{f\in\Hilbert_K(\Rd)}\frac{1}{N}\sum_{k=1}^{N}\max\left\{0,1-y_kf(x_k)\right\}+\lambda\norm{f}^2,
\]
where the binary data $\left\{(x_k,y_k):k\in\NN_{N}\right\}\subseteq\Rd\times\{\pm1\}$.
\begin{remark}
By employing the identical method of classical regularization,
we consider the simplistic form of the regularization terms, which rely exclusively on the regularization parameters and the norms of Banach spaces.
There exist numerous alternative formulas for regularization terms aimed at mitigating the risk of overfitting.
Regarding a conjecture, the regularization term $\lambda\regfun\left(\norm{f}\right)$ may be extended in a more general form
$\tilde{\regfun}:\Banach\times[0,\infty)^{M}\to[0,\infty)$ that satisfies the following conditions:
\begin{itemize}
\item[$\bullet$] For all $\lambda\in\RR_+^{M}$, $\tilde{\regfun}\left(f_1,\lambda\right)>\tilde{\regfun}\left(f_2,\lambda\right)$ if and only if $\norm{f_1}>\norm{f_2}$, $\tilde{\regfun}\left(f,\lambda\right)\to\infty$ when $\norm{f}\to\infty$, and $\tilde{\regfun}(\cdot,\lambda)$ is a weakly* lower semi-continuous function.
\item[$\bullet$] For all $f\in\Banach$, $\tilde{\regfun}(f,0)=0$,
    $\tilde{\regfun}\left(f,\lambda\right)/\norm{\lambda}_2\to\regfun\left(\norm{f}\right)$ when $\lambda\to0$,
    and $\tilde{\regfun}(f,\cdot)$ is a convex function.
\end{itemize}
By a similar argument, the same conclusions still hold true for the extended $\tilde{\regfun}\left(f,\lambda\right)$.
\end{remark}

Optimization \eqref{eq:RERM} is a primary formula in regularized learning.
Regularized learning is evidently founded techniques of regularization and
$\risk_n$ plays a crucial role. We will study the characteristics of $\risk_n$.

\begin{proposition}\label{Pro:EmpRisk}
\begin{itemize}
\item[(a)] Then
\[
\risk_n(f)=\risk_n(f+h),\quad\text{for all }f\in\Banach\text{ and all }h\in ker\left(\vxi_n^{\ast}\right).
\]
\item[(b)] If $\loss_n$ is a lower semi-continuous, continuous, or convex multi-loss function, then $\risk_n$ is weakly* lower semi-continuous, weakly* continuous, or convex.
\item[(c)] If $\Funset_{\Data}$ is relatively compact in $\Banach_{\ast}$ and $\Loss$ is uniformly local Lipschitz continuous, then $\left\{\risk_n:n\in\NN\right\}$ is weakly* equicontinuous on $r\Ball_{\Banach}$ for all $r>0$.
\item[(d)] Suppose that all conditions of (c) hold true. If $\risk_n$ converges pointwise to $\risk$ when $n\to\infty$, then
    \[
    \lim_{n\to\infty}\sup_{f\in r\Ball_{\Banach}}
    \abs{\risk(f)-\risk_n(f)}=0,\quad\text{for all }r>0.
    \]
\item[(e)] Suppose that all conditions of (c) hold true. Let $\left(f_{\alpha}\right)\subseteq\Banach$ be any weakly* convergent bounded net to an element $f_0$ and $\left(\lambda_{\alpha}\right)\subseteq\RR_{+}$ be any net. If
    \[
    \lim_{\alpha}
    \frac{1}{\lambda_{\alpha}}
    \sup_{\xi^{\ast}\in\Funset_{\Data}}
    \abs{\langle f_0-f_{\alpha},\xi^{\ast}\rangle}
    =0,
    \]
    then
    \[
    \lim_{\alpha}\frac{1}{\lambda_{\alpha}}\sup_{n\in\NN}
    \abs{\risk_{n}\left(f_0\right)-\risk_{n}\left(f_{\alpha}\right)}=0.
    \]
\end{itemize}
\end{proposition}

In the next section, we will explore the representer theorems, pseudo-approximation theorems, and convergence theorems of the approximate solutions.
Generally speaking, the convergence theorems ensure the approximation of regularized learning, and the representer theorems and pseudo-approximation theorems ensure the computation of regularized learning.

\subsection{Approximate Approximation}\label{sec:AppApp}

Let $\Nueral_m$ be a subset of $\Banach$ such that there exists a surjection $\Gamma_m$ from $\Omega_m$ onto $\Nueral_m$,
where $\Omega_m\subseteq\RR^m$ and $m\in\NN$.
For example, $\Nueral_m$ is a $m$ dimensional subspace, and $\Nueral_m$ is a collection of artificial neural networks with $m$ weights and biases despite the possibility of it not being a linear space.
Now, we consider an approximately approximate solution $f_{nm}^{\lambda}$ to minimize the regularized empirical risks over $\Nueral_m$, that is,
\begin{equation}\label{eq:RERM-App}
\risk_{n}\big(f_{nm}^{\lambda}\big)
+\lambda\regfun\big(\norm{f_{nm}^{\lambda}}\big)
=
\inf_{f\in\Nueral_m}
\risk_{n}(f)
+\lambda\regfun\left(\norm{f}\right).
\end{equation}
Let $\Sset_{n}^{\lambda}\left(\Nueral_m\right)$ be a collection of all minimizers of Optimization \eqref{eq:RERM-App}.
Thus, $f_{nm}^{\lambda}\in\Sset_{n}^{\lambda}\left(\Nueral_m\right)$.

According to Remark \ref{Rm:GenRepr}, if $\Banach_{\ast}$ is smooth, then the representer theorems ensure that there exists a $\Nueral_{N_n}$ such that $f_{n}^{\lambda}$ is equivalently solved by Optimization \eqref{eq:RERM-App}, such as support vector machines in RKHS and RKBS.
In numerous machine learning problems, the fast algorithms are developed by utilizing the specialized $\Nueral_m$.
In this case, $f_{n}^{\lambda}$ may not belong to any $\Nueral_m$.
Therefore, we introduce another condition of $\Nueral_m$ known as universal approximation.
Under the condition of universal approximation,
the pseudo-approximation theorems ensure that $f_{n}^{\lambda}$ is approximately solved by Optimization \eqref{eq:RERM-App}, such as artificial neural networks of using ReLU an KAN.

Moreover, we can solve a minimizer $\vw_m$ of the finite-dimensional optimization
\begin{equation}\label{eq:RERM-App-w}
\inf_{\vw\in\Omega_m}\loss_n\left(\vxi_n^{\ast},\vy_n,\langle\Gamma_m(\vw),\vxi_n^{\ast}\rangle\right)
+\lambda\regfun\left(\norm{\Gamma_m(\vw)}\right).
\end{equation}
This shows that $f_{nm}^{\lambda}=\Gamma_m\left(\vw_m\right)$.
According to the surjection from $\Omega_m$ onto $\Nueral_m$,
Optimization \eqref{eq:RERM-App-w} is an equivalent transformation of Optimization \eqref{eq:RERM-App}.

\subsection{Applications to Composite Algorithms}\label{sec:AppCompAlg}

In current researches of machine learning, numerous algorithms focus on addressing a singular problem.
The theory of regularized learning provides a straightforward approach to integrating multiple models and data into a unified system. Now, we will illustrate the concept with a generic example.

Let $\Banach^b,\Banach^w\subseteq\RR^{\Omega}$ be two Banach spaces and $\norm{\cdot}_b,\norm{\cdot}_w$ be their norms, respectively.
We assume that the exact solution $f^0\in\Banach^b\cap\Banach^w$ is solved by the both abstract models
\[
\inf_{f\in\Banach^b}\risk^b(f),\quad
\inf_{f\in\Banach^w}\risk^w(f),
\]
where $\risk^b:\Banach^b\to[0,\infty)$ and $\risk^w:\Banach^w\to[0,\infty)$ are two expected risk functions.
Moreover, $\risk^b,\risk^w$ are approximated by the empirical risk functions $\risk_n^b,\risk_n^w$, respectively. As shown in Equation \eqref{eq:EmpRiskFun}, $\risk_n^b,\risk_n^w$ are constructed by the training data $\big(\vxi^{b\ast}_{n},\vy_{n}^b\big),\big(\vxi^{w\ast}_{n},\vy_{n}^w\big)$ and multi-loss functions $\loss_n^b,\loss_n^w$, respectively.
Thus, $f^0$ is solved approximately by the both regaulairzed learning
\[
\inf_{f\in\Banach^b}\risk_n^b(f)+\lambda\norm{f}_b,\quad
\inf_{f\in\Banach^w}\risk_n^w(f)+\lambda\norm{f}_w.
\]

Let $\Banach:=\Banach^b\cap\Banach^w$ be a normed space endowed with the composite norm, that is,  $\norm{f}:=\norm{f}_{b}+\norm{f}_{w}$ for $f\in\Banach^b\cap\Banach^w$.
Obviously, $\Banach$ is a Banach space.
Since $\Banach^b\cap\Banach^w\cong\big(\Banach^b_{\ast}\big)^{\ast}\cap\big(\Banach^w_{\ast}\big)^{\ast}
\cong\big(\Banach^b_{\ast}+\Banach^w_{\ast}\big)^{\ast}$, we have
$\Banach^b_{\ast}+\Banach^w_{\ast}\cong\big(\Banach^b\cap\Banach^w\big)_{\ast}$.
Thus, $\Banach$ has a predual space $\Banach^b_{\ast}+\Banach^w_{\ast}$, and
$\xi^{\ast}\in\Banach_{\ast}$ if $\xi^{\ast}\in\Banach^b_{\ast}$ or $\xi^{\ast}\in\Banach^w_{\ast}$. This shows that the weak* convergence in $\sigma\left(\Banach,\Banach_{\ast}\right)$
implies the weak* convergence in $\sigma\big(\Banach^b,\Banach^b_{\ast}\big)$ and $\sigma\big(\Banach^w,\Banach^w_{\ast}\big)$.
We construct the composite models, that is,
\[
\risk:=\risk^b+\risk^w,\quad
\risk_n:=\risk_n^b+\risk_n^w.
\]
Thus, $f^0$ is a minimizer of the minimization of $\risk$ over $\Banach$, and $f^0$ is solved approximately by the composite regularized learning for minimizing $\risk_n+\lambda\norm{\cdot}$ over $\Banach$.
If $\risk_n^b,\risk_n^w$ both satisfy the conditions of the representer theorems, pseudo-approximation theorems, and convergence theorems in Section \ref{sec:ThmRegLearn},
then it is easy to check that $\risk_n$ satisfies the same conditions.
Therefore, the representer theorems, approximation theory, and convergence theorems hold true for the composite algorithms if they hold true for both individual algorithms.

Let $T^{b}_{n,\lambda}:=R_n^b+\lambda\norm{\cdot}_b$ and
$T^{w}_{n,\lambda}:=R_n^w+\lambda\norm{\cdot}_w$.
If $\risk_n^b$ and $\risk_n^w$ are weakly* lower semi-continuous and convex, then the minimizers of
the minimizations of $\risk_n^b+\lambda\norm{\cdot}_b$ and $\risk_n^w+\lambda\norm{\cdot}_w$ over $\Banach^b$ and $\Banach^w$ can be iteratively computed through the proximal point iteration, respectively, that is,
\[
g_{k+1}\in\text{prox}_{\theta T^{b}_{n,\lambda}}\left(g_k\right)\text{ or }g_{k+1}\in\text{prox}_{\theta T^{w}_{n,\lambda}}\left(g_k\right),
~k=0,1,\ldots,
\]
where $\text{prox}$ represents a proximal operator and $\theta>0$ is a step size.
Under the same hypotheses, the minimizer of the minimization of $\risk_n+\lambda\norm{\cdot}$ over $\Banach$
can be iteratively computed through the Douglas-Rachford splitting methods, that is,
\[
h^b_k\in\text{prox}_{\theta T^{b}_{n,\lambda}}\big(g_k\big),~
h^w_k\in\text{prox}_{\theta T^{w}_{n,\lambda}}\big(2h^b_k-g_k\big),~
g_{k+1}:=g_k+\sigma\big(h^w_k-h^b_k\big),
\]
$k=0,1,\ldots$, where $\sigma>0$ is a relaxation parameter.
This shows that the composite algorithm can be formulated as an alternating iterative system that integrates the both algorithms. Although composite algorithms are briefly discussed here, several unresolved issues remain.
In our forthcoming research, we intend to develop the theory of the composite algorithms that integrate data-driven and model-driven methods.

\section{Theorems of Regularized Learning}\label{sec:ThmRegLearn}

In the following theorems, the conditions of $\left(\vxi^{\ast}_{n},\vy_{n}\right)$ and $\loss_n$ are consistent with Proposition \ref{Pro:EmpRisk}. We will explore the characteristics of $f^0$, $f_{n}^{\lambda}$, and $f_{nm}^{\lambda}$,
including their existence, representation, and convergence.

\subsection{Representer Theorems}\label{sec:RepThm}

We improve the generalized representer theorems in Banach spaces in \cite{HuangLiuTanYe2020}
using alternative techniques from the unifying representer theorems for convex loss in \cite{Unser2020}.

\begin{lemma}\label{Lm:GenRepr}
If $\loss_n$ is a lower semi-continuous multi-loss function, then $\Sset_{n}^{\lambda}\left(\Banach\right)$ is nonempty and weakly* compact.
\end{lemma}

Specifically, if $\loss_n$ is a convex multi-loss function, which implies that $\risk_n$ is convex by Proposition \ref{Pro:EmpRisk} (b), then $\Sset_{n}^{\lambda}\left(\Banach\right)$ is also convex when $\Phi$ is convex.
Moreover, if $\loss_n$ is a convex multi-loss function and $\regfun$ and $\norm{\cdot}$ are strictly convex, then $\Sset_{n}^{\lambda}\left(\Banach\right)$ has at most one element.

\begin{theorem}\label{Thm:GenRepr}
If $\loss_n$ is a lower semi-continuous multi-loss function,
then for any $f_{n}^{\lambda}\in\Sset_{n}^{\lambda}\left(\Banach\right)$,
there exists a parameter $\vc_n\in\RR^{N_n}$ such that
\[
\text{(i) }\vc_n\cdot\vxi_n^{\ast}\in\partial\norm{\cdot}\big(f_{n}^{\lambda}\big),\quad
\text{(ii) }\vc_n\cdot\langle f_{n}^{\lambda},\vxi_n^{\ast}\rangle=\norm{f_{n}^{\lambda}}.
\]
\end{theorem}
The ``$\cdot$'' represents the dot product, such as $\vc\cdot\vxi^{\ast}=\sum_{k=1}^Nc_k\xi_k^{\ast}$,
where $\vc:=\left(c_1,c_2,\cdots,c_N\right)$ and $\vxi^{\ast}:=\left(\xi_1^{\ast},\xi_2^{\ast},\cdots,\xi_N^{\ast}\right)$.
\begin{remark}\label{Rm:GenRepr}
The parameter $\vc_{n}$ in Theorem \ref{Thm:GenRepr}
is dependent on the different minimizer $f_{n}^{\lambda}$.
If $\Banach_{\ast}$ is smooth, then \cite[Proposition 5.4.2]{Megginson1998} ensures that $\Sset_{n}^{\lambda}\left(\Banach\right)=\Sset_{n}^{\lambda}\left(\Nueral_{N_n}\right)$, where $\Nueral_{N_n}$ is a collection of all $f\in\Banach$ such that $\vc\cdot\vxi_n^{\ast}\in\partial\norm{\cdot}\left(f\right)$
for a $\vc\in\RR^{N_n}$.
Specifically, if $\Banach$ is a Hilbert space, then $f_{n}^{\lambda}\cong\norm{f_{n}^{\lambda}}\vc_n\cdot\vxi_n^{\ast}$.
If $f_{n}^{\lambda}\neq0$,
then the characterizing subdifferential of norm in Equation \eqref{eq:CharSubdiffNorm} shows that
$f_{n}^{\lambda}/\norm{f_{n}^{\lambda}}\in\partial\norm{\cdot}_{\ast}\left(\vc_n\cdot\vxi_n^{\ast}\right)$.
Let $\Eset_n(\va):=\left\{\langle f,\vxi_n^{\ast}\rangle:f\in\partial\norm{\cdot}_{\ast}\left(\va\cdot\vxi_n^{\ast}\right)\right\}$ for any $\va\in\RR^{N_n}$.
Thus, $f_{n}^{\lambda}$ can be solved as follows.
\begin{itemize}
\item[(i)]
Solving $r_n,\vv_n,\va_n$ of the finite-dimensional optimization
\[
\inf_{r\in\RR,\vv,\va\in\RR^{N_n}}\loss_n\left(\vxi_{n}^{\ast},\vy_{n},r\vv\right)+\lambda\regfun\left(r\right)
\text{ subject to }r\geq0,\vv\in\Eset_n(\va).
\]
\item[(ii)] Choosing $g_n\in\partial\norm{\cdot}_{\ast}\left(\va_n\cdot\vxi_n^{\ast}\right)$ subject to $\langle g_{n},\vxi_n^{\ast}\rangle=\vv_n$.
\item[(iii)] Taking $f_{n}^{\lambda}:=r_ng_n$.
\end{itemize}
Since $\vc_n=\va_n/\norm{\va_n\cdot\vxi_{n}^{\ast}}_{\ast}$ and
$r_n=\vc_n\cdot\langle f_{n}^{\lambda},\vxi_n^{\ast}\rangle$,
we also have $\norm{\va_n\cdot\vxi_{n}^{\ast}}_{\ast}=\va_n\cdot\vv_{n}$.
If $\loss_n$ is a linear combination of convex loss functions as shown in Equation \eqref{eq:lossfun-exa} and $\Banach$ is a 1-norm RKBS, the proximity algorithm can be utilized to solve $\vc_n$ and $f_{n}^{\lambda}$ by the similar techniques in \cite[Theorem 64]{WangXu2021}.
In our forthcoming papers, we will study various numerical algorithms for specialized loss functions and Banach spaces, such as alternating direction methods of multipliers and composite optimization methods.
\end{remark}

If $\loss_n$ is a convex multi-loss function and $\Phi$ is a convex function, then
the Krein-Milman theorem ensures that $\Sset_{n}^{\lambda}\left(\Banach\right)$ is a closed convex hull of its extreme points.
If $\Banach$ is strictly convex, then for any $f_{n}^{\lambda}\neq0$, the element $f_{n}^{\lambda}/\norm{f_{n}^{\lambda}}$ is an extreme point of $\Ball_{\Banach}$.
Next, we will verify that there exists a $\tilde{f}_{n}^{\lambda}$, which is
a linear combination of finitely many extreme points of $\Ball_{\Banach}$.
\begin{corollary}\label{Cor:GenReprExtPoint}
If $\loss_n$ is a lower semi-continuous multi-loss function,
then for any $f_{n}^{\lambda}\in\Sset_{n}^{\lambda}\left(\Banach\right)$,
there exists a $\tilde{f}_{n}^{\lambda}\in\Sset_{n}^{\lambda}\left(\Banach\right)$ such that
\[
\text{(i) }\tilde{f}_{n}^{\lambda}=\vb_n\cdot\ve_n,\quad
\text{(ii) }\langle \tilde{f}_{n}^{\lambda},\vxi_n^{\ast}\rangle=\langle f_{n}^{\lambda},\vxi_n^{\ast}\rangle,
\]
where $M_n\leq N_n$,
the parameter $\vb_n\in\RR^{M_n}$, and
$\ve_n:=\left(e_{n1},e_{n2},\cdots,e_{nM_n}\right)$ composes of
extreme points $e_{n1},e_{n2},\ldots,e_{nM_n}$ of $\Ball_{\Banach}$.
\end{corollary}

The generalized representer theorems ensure that the approximate solutions can be computed by the numerical algorithms driven from the equivalent finite-dimensional optimizations for different appropriate learning models.
We can also verify that the generalized representer theorems encompass the classical representer theorems in RKHS and RKBS.
More detailed information regarding the representer theorems is mentioned in \cite{HuangLiuTanYe2020,SteinwartChristmann2008,Unser2020,UnserFageotWard2017,WangXu2021,XuYe2019,ZhangXuZhang2009}.

\subsection{Pseudo-approximation Theorems}\label{sec:AppThm}

In the beginning, we study the existence of $f_{nm}^{\lambda}$ of Optimization \eqref{eq:RERM-App}.

\begin{lemma}\label{Lm:App}
Suppose that nonempty $\Nueral_m$ is weakly* closed.
If $\loss_n$ is a lower semi-continuous multi-loss function, then $\Sset_{n}^{\lambda}\left(\Nueral_m\right)$
is nonempty and weakly* compact.
\end{lemma}

\begin{remark}
If $\Nueral_m$ is not weakly* closed, then Optimization \eqref{eq:RERM-App} or \eqref{eq:RERM-App-w}
may not exist a minimizer.
Substituting $\Nueral_m$ into the weak* closure of $\Nueral_m$, we can reconstruct Optimization \eqref{eq:RERM-App}
for the existence of $f_{nm}^{\lambda}$ if there exists an original minimizing net which weakly* converges to the minimizer.
Therefore, we can still solve Optimization \eqref{eq:RERM-App-w} approximately to obtain the estimators of $f_{nm}^{\lambda}$
through calculus of variations, even when Optimization \eqref{eq:RERM-App-w} does not exist a minimizer.
\end{remark}

We say that $\left\{\Nueral_m:m\in\NN\right\}$ satisfies \emph{universal approximation} in $\Banach$
if $\Nueral_m\subseteq\Nueral_{m+1}$ for all $m\in\NN$, and for any $f\in\Banach$ and any $\epsilon>0$, there exists a $g_{m_{\epsilon}}\in\Nueral_{m_{\epsilon}}$ such that $\norm{f-g_{m_{\epsilon}}}\leq\epsilon$.
The universal approximation also implies that for any $f\in\Banach$, there exists a sequence $\left(g_{m_k}\right)$ in an increasing order $\left(m_k\right)$ such that $g_{m_k}\in\Nueral_{m_k}$ for all $k\in\NN$ and $\norm{f-g_{m_k}}\to0$ when $k\to\infty$.
This shows that the universal approximation based on $\Nueral_m$ is the same as the classical universal approximation based on reproducing kernels in \cite{MicchelliXuZhang2006}.
Incidentally, $\Sset_{n}^{\lambda}\left(\Nueral_m\right)$ may have various elements.
For any $m\in\NN$,
we actually need a minimizer from $\Sset_{n}^{\lambda}\left(\Nueral_m\right)$ to approximate $f_{n}^{\lambda}$.
Thus, $f_{nm}^{\lambda}$ will be an \emph{arbitrarily} chosen element, when $m$ is given,
and the choices of $f_{nm}^{\lambda}$ will not affect the conclusions of the pseudo-approximation theorems.
Let $\big(f_{nm}^{\lambda}\big)$ be a net with a directed set $\NN$ in its usual order, that is, $m_1\succeq m_2$ if $m_1\geq m_2$.

\begin{theorem}\label{Thm:App}
Suppose that nonempty $\Nueral_m$ is weakly* closed for all $m\in\NN$.
Further suppose that $\cap_{m\in\NN}\Nueral_m$ is nonempty and
$\left\{\Nueral_m:m\in\NN\right\}$ satisfies universal approximation in $\Banach$.
If $\loss_n$ is a continuous multi-loss function, then there exists a weakly* convergent bounded subnet $\big(f_{nm_{\alpha}}^{\lambda}\big)$ of $\big(f_{nm}^{\lambda}\big)$ to an element $f_n^{\lambda}\in\Sset_{n}^{\lambda}\left(\Banach\right)$ such that
\[
\text{(i) }\risk_n\big(f_n^{\lambda}\big)+\lambda\regfun\big(\norm{f_n^{\lambda}}\big)
=\lim_{\alpha}\risk_n\big(f_{nm_{\alpha}}^{\lambda}\big)+\lambda\regfun\big(\norm{f_{nm_{\alpha}}^{\lambda}}\big),
\]
and
\[
\text{(ii) }\norm{f_n^{\lambda}}=\lim_{\alpha}\norm{f_{nm_{\alpha}}^{\lambda}}.
\]
\end{theorem}

\subsection{Convergence Theorems}\label{sec:ConThm}

Now, we will study convergence analysis of the approximate solutions.
The regularized learning relies on the pointwise convergence as shown in Equation \eqref{eq:PointConv}.
However, the pointwise convergence may not be sufficient to verify the convergence of the approximate solutions to the exact solutions.
Accordingly, the proof requires the inclusion of the additional conditions of $\Data$ and $\Loss$.
Indeed, Proposition \ref{Pro:EmpRisk} (c) provides the specific conditions of $\Data$ and $\Loss$.
In Theorems \ref{Thm:ConvThm}, \ref{Thm:ConvThm-subnet}, and \ref{Thm:ConvThm-lambda},
we therefore assume that
\begin{itemize}
\item[(I)] $\risk_n$ converges pointwise to $\risk$ when $n\to\infty$,
\item[(II)] $\Funset_{\Data}$ is relatively compact in $\Banach_{\ast}$ and $\Loss$ is uniformly local Lipschitz continuous.
\end{itemize}
Condition (I) is dependent on $\risk$, while Condition (II) is independent on $\risk$.
Roughly speaking, Condition (I) is regarded as the assumption of models,
while Condition (II) is regarded as the assumption of data.
In practical applications, even if $\risk$ is frequently unidentified or uncertain,
then Condition (I) can be checked by employing test functions,
and Condition (II) can be checked through the utilization of linear-functional data and loss functions.
According to Proposition \ref{Pro:EmpRisk} (c), Condition (II) implies that
\begin{itemize}
\item[(II$'$)] $\left\{\risk_n:n\in\NN\right\}$ is weakly* equicontinuous on $r\Ball_{\Banach}$ for all $r>0$.
\end{itemize}
By a similar argument in Remark \ref{Rm:Equicontin-data}, Condition (II$'$) can be interpreted as the intermediate condition between equicontinuity and weak* equicontinuity.
On substituting Condition (II) into (II$'$), Theorem~\ref{Thm:ConvThm} still holds true,
as explained in Remark \ref{Rm:Gamma-Conv}.

Since the local Lipschitz continuity of $\loss_n$ implies the lower semi-continuity of $\loss_n$,
Lemma \ref{Lm:GenRepr} ensures that $\Sset_{n}^{\lambda}\left(\Banach\right)$ is nonempty for all $n\in\NN$ and all $\lambda>0$.
Incidentally, $\Sset_{n}^{\lambda}\left(\Banach\right)$ may have various elements.
For any $n\in\NN$ and any $\lambda>0$,
we actually need a minimizer from $\Sset_{n}^{\lambda}\left(\Banach\right)$ to approximate $f^0$.
Specifically, when $n,\lambda$ are given, then $f_{n}^{\lambda}$ will be an \emph{arbitrarily} chosen element, and the choices of $f_{n}^{\lambda}$ will not affect the conclusions of the convergence theorems.
Let $\big(f_n^{\lambda}\big)$ be a net
with a directed set
$\NN\times\RR_{+}$ such that
$(n_1,\lambda_1)\succeq(n_2,\lambda_2)$ if $n_1\geq n_2$ and $\lambda_1\leq\lambda_2$.
Thus, $\lim_{(n,\lambda)}(n,\lambda)=(\infty,0)$.

\begin{theorem}\label{Thm:ConvThm}
Suppose that Conditions (I) and (II) hold true.
If $\Sset^{0}\left(\Banach\right)$ is nonempty, then there exists a weakly* convergent bounded subnet $\big(f_{n_{\alpha}}^{\lambda_{\alpha}}\big)$ of $\big(f_n^{\lambda}\big)$ to an element $f^0\in\Sset^0\left(\Banach\right)$ such that
\[
\text{(i) }\risk\big(f^0\big)=\lim_{\alpha}\risk\big(f_{n_{\alpha}}^{\lambda_{\alpha}}\big),\quad
\text{(ii) }\norm{f^0}=\lim_{\alpha}\norm{f_{n_{\alpha}}^{\lambda_{\alpha}}},
\]
and
\[
\text{(iii) }\norm{f^0}=\inf\left\{\norm{\tilde{f}^0}:
\tilde{f}^0\in\Sset^0\left(\Banach\right)\right\}.
\]
\end{theorem}
\begin{remark}
Lemma \ref{Lma:net-conv} ensures that
\[
\risk\big(f^0\big)
=
\lim_{\alpha}\risk_{n_{\alpha}}\big(f_{n_{\alpha}}^{\lambda_{\alpha}}\big).
\]
According to Lemma \ref{Lm:CompactEquicontinuous}, we have
\[
\lim_{\alpha}\sup_{\xi^{\ast}\in\Funset}
\abs{\langle f^0-f_{n_{\alpha}}^{\lambda_{\alpha}},\xi^{\ast}\rangle}=0,
\]
where $\Funset$ is any relatively compact set in $\Banach_{\ast}$.
The (iii) also shows that $f^0$ is a minimum-norm minimizer of Optimization \eqref{eq:TRM}.
This indicates that the ``best'' approximate solution could be obtained through regularized learning.
\end{remark}

Next, we study any weakly* convergent bounded subnet of $\big(f_n^{\lambda}\big)$.
\begin{theorem}\label{Thm:ConvThm-subnet}
Suppose that Conditions (I) and (II) hold true.
If
$\big(f_{n_{\alpha}}^{\lambda_{\alpha}}\big)$ is a weakly* convergent bounded subnet of $\big(f_n^{\lambda}\big)$ to an element $f^0$,
then $f^0\in\Sset^0\left(\Banach\right)$ such that
\[
\text{(i) }\risk\big(f^0\big)=\lim_{\alpha}\risk\big(f_{n_{\alpha}}^{\lambda_{\alpha}}\big).
\]
Moreover, if
\begin{equation}\label{eq:add-cond-weak-net}
\lim_{\alpha}
\frac{1}{\lambda_{\alpha}}
\sup_{\xi^{\ast}\in\Funset_{\Data}}
\abs{\langle f^0-f_{n_{\alpha}}^{\lambda_{\alpha}},\xi^{\ast}\rangle}
=0,
\end{equation}
then
\[
\text{(ii) }\norm{f^0}=\lim_{\alpha}\norm{f_{n_{\alpha}}^{\lambda_{\alpha}}}.
\]
\end{theorem}
\begin{remark}
Since $\Banach_{\ast}$ is complete, any weakly* convergent sequence of
$\big(f_n^{\lambda}\big)$ is bounded. This shows that the condition of boundedness in Theorem \ref{Thm:ConvThm-subnet} is unnecessary, when the subnet in Theorem \ref{Thm:ConvThm-subnet} is a sequence.
Since
\[
\lim_{\alpha}\sup_{\xi^{\ast}\in\Funset_{\Data}}
\abs{\langle f^0-f_{n_{\alpha}}^{\lambda_{\alpha}},\xi^{\ast}\rangle}=0,
\quad
\lim_{\alpha}\lambda_{\alpha}=0,
\]
the condition in Equation \eqref{eq:add-cond-weak-net} illustrates that
the regularization parameters do not decrease faster than the errors at the training inputs.
\end{remark}

As shown in the proofs of Theorems \ref{Thm:ConvThm} and \ref{Thm:ConvThm-subnet}, if $\Banach_{\ast}$ is separable or $\Banach$ is reflexive,
then $\Ball_{\Banach}$ is weakly* sequentially compact. Thus, the weakly* convergent subnets discussed above can be substituted into the weakly* convergent sequences.
Finally, we consider a specific sequence $\big(f_{n}^{\lambda_n}\big)$ .
We will study the well-posed adaptive regularization parameter $\lambda$ of Optimization \eqref{eq:TRM}.
If $\risk\left(f^0\right)=0$, then the pointwise convergence in
Equation \eqref{eq:PointConv} shows that
$\risk_n\left(f^0\right)\to0$ when $n\to\infty$.
Thus, there exists a decrease sequence $\left(\lambda_n\right)$ to $0$
such that
$\lambda_n^{-1}\risk_n\left(f^0\right)\to0$ when $n\to\infty$.
Obviously, the sequence $\big(f_{n}^{\lambda_n}\big)$ is a subnet of $\big(f_n^{\lambda}\big)$.
\begin{theorem}\label{Thm:ConvThm-lambda}
Suppose that Conditions (I) and (II) hold true.
If $\Sset^0\left(\Banach\right)=\left\{f^0\right\}$
and $\left(\lambda_n\right)$ is a decrease sequence to $0$ such that
\begin{equation}\label{eq:add-cond-strong-seq}
\lim_{n\to\infty}\frac{\risk_n\left(f^0\right)}{\lambda_n}=0,
\end{equation}
then
$\big(f_{n}^{\lambda_n}\big)$ is a weakly* convergent bounded sequence to $f^0$
such that
\[
\text{(i) }\risk\big(f^0\big)=\lim_{n\to\infty}\risk\big(f_{n}^{\lambda_n}\big),\quad
\text{(ii) }\norm{f^0}=\lim_{n\to\infty}\norm{f_{n}^{\lambda_n}}.
\]
\end{theorem}
\begin{remark}
If $\Banach$ is a reflexive Radon-Riesz space, then
Theorem \ref{Thm:ConvThm-lambda} also ensures that
\[
\lim_{n\to\infty}\norm{f^0-f_{n}^{\lambda_n}}=0.
\]
Equation \eqref{eq:add-cond-strong-seq} implies that $\risk_n\left(f^0\right)\to0$ when $n\to\infty$.
Thus, $\risk\left(f^0\right)=0$.
In the proof of Theorem \ref{Thm:ConvThm-lambda}, the condition in Equation \eqref{eq:add-cond-strong-seq} is utilized to guarantee the boundedness of $\big(f_{n}^{\lambda_n}\big)$.
But, the condition in Equation \eqref{eq:add-cond-strong-seq} is still needed even when $\big(f_{n}^{\lambda_n}\big)$ is bounded.
\end{remark}

In this article, we study the convergence of the approximate solutions using weak* topology. We already know that a weakly* convergent net may not be bounded. Hence, any globally convergent rate is not taken into consideration in the convergence theorems.
Specifically, according to \cite[Theorem~3.6]{Ye2019I} and \cite[Theorem~3.7]{Ye2019II},
if $\Banach$ is endowed with a kernel-based probability measure and $\loss_n$ is formulated using a least-squared loss,
then given a specific exact solution $f^0$, we have $\abs{\langle f^0-f_{n}^{\lambda},\xi^{\ast}\rangle}\leq\sigma_{\xi^{\ast}\mid\vxi_{n}^{\ast},\vy_{n}}+\kappa_{\lambda\Phi\mid\vxi_{n}^{\ast},\vy_{n}}$,
for any $\xi^{\ast}\in\Banach_{\ast}$,
where $\sigma_{\xi^{\ast}\mid\vxi_{n}^{\ast},\vy_{n}}>0$ is associated with the testing input conditioned on the training data, and $\kappa_{\lambda\Phi\mid\vxi_{n}^{\ast},\vy_{n}}>0$ is associated with the regularization term
conditioned on the training data. Therefore, we conjecture that a locally convergent rate could be achieved
for some algorithms of machine learning, such as physics-informed neural networks (PINNs).

\section{Proofs of Theorems}\label{sec:PrfThm}

In the beginning, we prove Lemmas \ref{Lm:CompactEquicontinuous}, \ref{Lm:weak-equicont}, \ref{Lm:loss}, and \ref{Lm:Phi}.
In this article, some lemmas may have been previously addressed in other monographs.
However, they do not precisely align with the references.
Due to the rigorous discussions, we will prove all lemmas ourselves.
We emphasize that $\Cont\left(\Ball_{\Banach}\right)$ is a continuous function space defined on $\Ball_{\Banach}$
with respect to the relative weak* topology.

\begin{proof}[{\bf Proof of Lemma \ref{Lm:CompactEquicontinuous}}]
Since the elements of $\Funset$ are linear functionals, it is obvious that (a) and (b) are equivalent.
Next, if we prove that (a) and (c) are equivalent, then the proof is complete.
Since $\norm{\xi^{\ast}}_{\ast}=\sup_{f\in\Ball_{\Banach}}\abs{\langle f,\xi^{\ast}\rangle}$ for $\xi^{\ast}\in\Funset$,
$\Funset$ is relatively compact in $\Banach_{\ast}$ if and only if $\Funset$ is relatively compact in $\Cont\left(\Ball_{\Banach}\right)$. Moreover, $\Funset$ is weakly* equicontinuous on $\Ball_{\Banach}$ if and only if
$\Funset$ is equicontinuous in $\Cont\left(\Ball_{\Banach}\right)$.
If $\Funset$ is equicontinuous in $\Cont\left(\Ball_{\Banach}\right)$,
\cite[Corollary~4.1]{SchaeferWolff1999} ensures that $\Funset$ is bounded in $\Cont\left(\Ball_{\Banach}\right)$ by its linearity.
Thus, the Arzel\'a-Ascoli theorem ensures that $\Funset$ is equicontinuous in $\Cont\left(\Ball_{\Banach}\right)$ if and only if
$\Funset$ is relatively compact in $\Cont\left(\Ball_{\Banach}\right)$.
This demonstrates that (a) and (c) are equivalent.
\end{proof}

\begin{proof}[{\bf Proof of Lemma \ref{Lm:weak-equicont}}]
According to Lemma \ref{Lm:CompactEquicontinuous}, we only need to prove that $\left\{\delta_{x_n}:n\in\NN\right\}$ is weakly* equicontinuous on $\Ball_{\Banach}$.
We take any weak* convergent net $\left(f_{\alpha}\right)\subseteq\Ball_{\Banach}$ to an element $f_0\in\Ball_{\Banach}$.
By using the embedding property, we have
\begin{equation}\label{eq:weak-equicont-1}
\abs{\langle f,\delta_{x}-\delta_{z}\rangle}=\abs{f(x)-f(z)}\leq\norm{f}_{\Cont^{0,\vartheta}(\Omega)}\norm{x-z}_2^{\vartheta}\leq C\norm{f}\norm{x-z}_2^{\vartheta},
\end{equation}
for any $x,z\in\Omega$ and any $f\in\Banach$,
where the constant $C>0$ is independent on $x,z,f$.
Next, we take any $\epsilon>0$. Let $\theta:=\left(\epsilon/(3C)\right)^{1/\vartheta}$.
The collection of closed balls $z+\theta\Ball_{\Rd}$ for all $z\in\Omega$ forms a cover of $\Omega$.
Since $\Omega$ is compact, this cover has a finite subcover
$z_k+\theta\Ball_{\Rd}$ for $k\in\NN_N$.
For any $n\in\NN$, there exists a $k_n\in\NN_N$ such that $x_n\in z_{k_n}+\theta\Ball_{\Rd}$.
Since $x_n-z_{k_n}\in\theta\Ball_{\Rd}$ and $f_0,f_{\alpha}\in\Ball_{\Banach}$,
Equation \eqref{eq:weak-equicont-1} shows that
\begin{equation}\label{eq:weak-equicont-2}
\abs{\langle f_{0},\delta_{x_n}-\delta_{z_{k_n}}\rangle}\leq C\theta^{\vartheta}\leq\frac{\epsilon}{3}.
\quad
\abs{\langle f_{\alpha},\delta_{x_n}-\delta_{z_{k_n}}\rangle}\leq C\theta^{\vartheta}\leq\frac{\epsilon}{3}.
\end{equation}
Moreover, since $\langle f_{0},\delta_{z_k}\rangle=\lim_{\alpha}\langle f_{\alpha},\delta_{z_k}\rangle$ for all $k\in\NN_N$,
there exists a $\gamma$ such that
\begin{equation}\label{eq:weak-equicont-3}
\abs{\langle f_{0}-f_{\alpha},\delta_{z_k}\rangle}
\leq\frac{\epsilon}{3}\text{ for all }k\in\NN_N,\quad\text{when }\alpha\succeq\gamma.
\end{equation}
We conclude from Equations \eqref{eq:weak-equicont-2} and \eqref{eq:weak-equicont-3} that
\[
\begin{split}
&\abs{\langle f_{0}-f_{\alpha},\delta_{x_n}\rangle}\\
\leq&
\abs{\langle f_{0},\delta_{x_n}\rangle-\langle f_{0},\delta_{z_{k_n}}\rangle}
+
\abs{\langle f_{0},\delta_{z_{k_n}}\rangle-\langle f_{\alpha},\delta_{z_{k_n}}\rangle}
+
\abs{\langle f_{\alpha},\delta_{z_{k_n}}\rangle-\langle f_{\alpha},\delta_{x_n}\rangle}
\\
\leq&\frac{\epsilon}{3}+\frac{\epsilon}{3}+\frac{\epsilon}{3}=\epsilon,
\quad{ }\text{when }\alpha\succeq\gamma.
\end{split}
\]
This shows that
\[
\lim_{\alpha}\sup_{n\in\NN}\abs{\langle f_{0}-f_{\alpha},\delta_{x_n}\rangle}=0.
\]
Therefore, Lemma \ref{Lm:CompactEquicontinuous} ensures that $\left\{\delta_{x_n}:n\in\NN\right\}$ is relatively compact in $\Banach_{\ast}$.
\end{proof}

\begin{proof}[{\bf Proof of Lemma \ref{Lm:loss}}]
Since $\loss_n\left(\vxi_n,\vy_n,\cdot\right)$ is a linear combination of $\closs\left(\xi_1,y_1,\cdot\right)$, $\closs\left(\xi_2,y_2,\cdot\right)$, ..., $\closs\left(\xi_{N_n},y_{N_n},\cdot\right)$,
the proof of (a) is straightforward.
As easy computation shows that $C_{n}^{\theta}=C_{\theta}$ for all $n\in\NN$, where $C_{\theta}$ is the local Lipschitz constant of $L$ for any $\theta>0$. Thus, (b) holds true.
\end{proof}

It is obvious that $\regfun(r_1)<\regfun(r_2)$ if and only if $r_1<r_2$ for any $r_1,r_2\in[0,\infty)$.

\begin{lemma}\label{Lm:Phi}
For any net $\left(r_{\alpha}\right)\subseteq[0,\infty)$ and $r_0\in[0,\infty)$,
$\regfun\left(r_{0}\right)=\lim_{\alpha}\regfun\left(r_{\alpha}\right)$ if and only if
$r_{0}=\lim_{\alpha}r_{\alpha}$.
\end{lemma}
\begin{proof}
Since $\regfun$ is strictly increasing and continuous,
the inverse of $\regfun$ exists and $\regfun^{-1}$ is also strictly increasing and continuous.
Thus, the proof is straightforward.
\end{proof}

\subsection{Proofs of Characteristics of $\risk_n$}\label{sec:ProofPropRn}

We will prove Proposition \ref{Pro:EmpRisk}, when $R_n$ has the formula as shown in Equation \eqref{eq:EmpRiskFun}.

\begin{proof}[{\bf Proof of Proposition \ref{Pro:EmpRisk}}]
Since
\[
\loss_n\left(\vxi_{n}^{\ast},\vy_{n},\langle f,\vxi_{n}^{\ast}\rangle\right)=
\loss_n\left(\vxi_{n}^{\ast},\vy_{n},\langle f+h,\vxi_{n}^{\ast}\rangle\right),
\]
for all $f\in\Banach$ and all $h\in
ker\left(\vxi_n^{\ast}\right)$, the (a) holds true.

Next, we prove the (b). If $\loss_n$ is a lower semi-continuous multi-loss function,
then $\loss_n\left(\vxi_{n}^{\ast},\vy_{n},\cdot\right)$ is lower semi-continuous; hence
it is easy to check that $\risk_n$ is weakly* lower-semi continuous.
Similar arguments apply to the cases of continuity and convexity.

Moreover, we prove the (c).
Let $\left(f_{\alpha}\right)\subseteq r\Ball_{\Banach}$ be any weakly* convergent net to an element $f_0\in r\Ball_{\Banach}$ for a $r>0$.
Since $\Funset_D$ is relatively compact in $\Banach_{\ast}$,
there exists a $\varrho>0$ such that $\norm{\xi^{\ast}}_{\ast}\leq\varrho$ for all $\xi^{\ast}\in\Funset_D$.
Thus, we have
\[
\abs{\langle f_{0},\xi^{\ast}\rangle}
\leq\norm{\xi^{\ast}}_{\ast}\norm{f_{0}}
\leq\varrho r,
\quad
\abs{\langle f_{\alpha},\xi^{\ast}\rangle}
\leq\norm{\xi^{\ast}}_{\ast}\norm{f_{\alpha}}
\leq\varrho r,
\]
for all $\xi^{\ast}\in\Funset_D$. We take $\theta:=\varrho r$.
Thus,
\[
\norm{\langle f_{0},\vxi_n^{\ast}\rangle}_{\infty}\leq\theta,
\quad
\norm{\langle f_{\alpha},\vxi_n^{\ast}\rangle}_{\infty}\leq\theta,
\quad\text{for all }n\in\NN.
\]
Moreover, since $\Loss$ is uniformly local Lipschitz continuous, there exists a constant $C_{\theta}:=\sup_{n\in\NN}C_n^{\theta}$ such that
\[
\abs{
\loss_n\left(\vxi_n^{\ast},\vy_n,\langle f_0,\vxi_n^{\ast}\rangle\right)
-\loss_n\left(\vxi_n^{\ast},\vy_n,\langle f_{\alpha},\vxi_n^{\ast}\rangle\right)
}
\leq C_{\theta}\norm{\langle f_0,\vxi_n^{\ast}\rangle-\langle f_{\alpha},\vxi_n^{\ast}\rangle}_{\infty},
\]
for all $\left(\vxi_n^{\ast},\vy_n\right)\in\Data$, where $C_n^{\theta}$ is the local Lipschitz constant of $\loss_n$.
Thus,
\begin{equation}\label{eq:Ln-equiCont-1}
\abs{\risk_n\left(f_0\right)-\risk_n\left(f_{\alpha}\right)}
\leq C_{\theta}\norm{\langle f_0,\vxi_n^{\ast}\rangle-\langle f_{\alpha},\vxi_n^{\ast}\rangle}_{\infty}.
\end{equation}
Lemma \ref{Lm:CompactEquicontinuous} ensures that $\Funset_D$ is weakly* equicontinuous on $\left(f_{\alpha}\right)\subseteq r\Ball_{\Banach}$. Therefore, we have
\[
\lim_{\alpha}\sup_{n\in\NN}\abs{\risk_n\left(f_0\right)-\risk_n\left(f_{\alpha}\right)}
\leq C_{\theta}\lim_{\alpha}\sup_{\xi^{\ast}\in\Funset_D}
\abs{\langle f_0,\xi^{\ast}\rangle-\langle f_{\alpha},\xi^{\ast}\rangle}=0.
\]
This demonstrates that $\left\{\risk_n:n\in\NN\right\}$ is weakly* equicontinuous on $r\Ball_{\Banach}$.

We will now the (d). According to the (c), $\left\{\risk_n:n\in\NN\right\}$ is also equicontinuous in $\Cont\left(r\Ball_{\Banach}\right)$.
By the pointwise convergence in Equation \eqref{eq:PointConv},
the Arzel\'a-Ascoli theorem ensures that $\risk_n$ uniformly converges to $\risk$ on $r\Ball_{\Banach}$.

Finally, we prove the (e). Equation \eqref{eq:Ln-equiCont-1} shows that
\[
\lim_{\alpha}\frac{1}{\lambda_{\alpha}}\sup_{n\in\NN}\abs{\risk_n\left(f_0\right)-\risk_n\left(f_{\alpha}\right)}
\leq C_{\theta}\lim_{\alpha}\frac{1}{\lambda_{\alpha}}\sup_{\xi^{\ast}\in\Funset_D}
\abs{\langle f_0-f_{\alpha},\xi^{\ast}\rangle}=0.
\]
\end{proof}

\subsection{Proofs of Representer Theorems}\label{sec:ProofReprThm}

In this subsection, we fix $n,\lambda$ in the proof of Lemma \ref{Lm:GenRepr}, Theorem \ref{Thm:GenRepr}, and Corollary \ref{Cor:GenReprExtPoint}.
Since $\loss_n$ is a lower semi-continuous multi-loss function, Proposition \ref{Pro:EmpRisk} (b) ensures that
$\risk_n$ is weakly* lower semi-continuous.

\begin{proof}[{\bf Proof of Lemma \ref{Lm:GenRepr}}]
We first prove that $\Sset_{n}^{\lambda}\left(\Banach\right)$ is nonempty.
Since $\regfun(r)\to\infty$ when $r\to\infty$, we have
$\risk_n(f)+\lambda\regfun\left(\norm{f}\right)\to\infty$
when $\norm{f}\to\infty$.
Thus, there exists a $r>0$ such that
\[
\inf_{f\in\Banach}\risk_n(f)+\lambda\regfun\left(\norm{f}\right)
=
\inf_{f\in r\Ball_{\Banach}}\risk_n(f)+\lambda\regfun\left(\norm{f}\right).
\]
Since $r\Ball_{\Banach}$ is weakly* compact and $\risk_n+\lambda\regfun\left(\norm{\cdot}\right)$ is
weakly* lower semi-continuous,
the Weierstrass extreme value theorem ensures that
$\risk_n+\lambda\regfun\left(\norm{\cdot}\right)$ attains a global minimum on $r\Ball_{\Banach}$ and thus on $\Banach$. Therefore $\Sset_{n}^{\lambda}\left(\Banach\right)\neq\emptyset$.

Next, for any $f\in\Sset_{n}^{\lambda}\left(\Banach\right)$, we have
\[
\lambda\regfun\big(\norm{f}\big)\leq
\risk_n\left(f\right)+\lambda\regfun\left(\norm{f}\right)
\leq\risk_n\left(0\right)+\lambda\regfun\left(0\right).
\]
Thus,
\[
\norm{f}\leq\regfun^{-1}\left(\frac{\risk_n\left(0\right)}{\lambda}+\regfun\left(0\right)\right).
\]
This demonstrates that $\Sset_{n}^{\lambda}\left(\Banach\right)$ is bounded.
Therefore, if we prove that $\Sset_{n}^{\lambda}\left(\Banach\right)$ is weakly* closed,
then the Banach-Alaoglu theorem ensures that $\Sset_{n}^{\lambda}\left(\Banach\right)$ is weakly* compact.
Let $\left(f_{\alpha}\right)\subseteq\Sset_{n}^{\lambda}\left(\Banach\right)$ be any weakly* convergent net to an element $f_0\in\Banach$.
Thus,
\[
\liminf_{\alpha}\risk_n\left(f_{\alpha}\right)+\lambda\regfun\left(\norm{f_{\alpha}}\right)
=
\inf_{f\in\Banach}
\risk_{n}(f)+\lambda\regfun\left(\norm{f}\right).
\]
Since $\risk_n+\lambda\regfun\left(\norm{\cdot}\right)$ is
weakly* lower semi-continuous,
we have
\[
\risk_n\left(f_{0}\right)+\lambda\regfun\left(\norm{f_{0}}\right)
\leq
\liminf_{\alpha}\risk_n\left(f_{\alpha}\right)+\lambda\regfun\left(\norm{f_{\alpha}}\right).
\]
Thus,
\[
\risk_n\left(f_{0}\right)+\lambda\regfun\left(\norm{f_{0}}\right)\leq\inf_{f\in\Banach}
\risk_{n}(f)+\lambda\regfun\left(\norm{f}\right).
\]
This demonstrates that $f_0\in\Sset_{n}^{\lambda}\left(\Banach\right)$.
Therefore, $\Sset_{n}^{\lambda}\left(\Banach\right)$ is weakly* closed.
\end{proof}

Next, we will prove Theorem \ref{Thm:GenRepr} using Lemmas \ref{Lm:GenRepr-FiniteRep} and \ref{Lm:GenRepr-AS}.
For any $\vt\in\RR^{N_n}$, we denote an affine space
\[
\Aset(\vt):=\left\{f\in\Banach:\langle f,\vxi_n^{\ast}\rangle=\vt\right\},
\]
where $\vxi_n^{\ast}$ is given in Theorem \ref{Thm:GenRepr}.
We consider the optimization
\begin{equation}\label{eq:MinNormIntr}
\inf_{f\in\Aset(\vt)}\norm{f}.
\end{equation}
Let $\Iset(\vt)$ be a collection of all minimizers of Optimization \eqref{eq:MinNormIntr}.
Obviously, if $\Aset(\vt)\neq\emptyset$, then $\Iset(\vt)\neq\emptyset$.

\begin{lemma}\label{Lm:GenRepr-FiniteRep}
If $\Aset(\vt)$ is nonempty for a fixed $\vt\in\RR^{N_n}$,
then for any $f_{\vt}\in\Iset(\vt)$, there exists a $\vc_n\in\RR^{N_n}$ such that
\[
\text{(i) }\vc_n\cdot\vxi_n^{\ast}\in\partial\norm{\cdot}\left(f_{\vt}\right),\quad
\text{(ii) }\vc_n\cdot\langle f_{\vt},\vxi_n^{\ast}\rangle=\norm{f_{\vt}}.
\]
\end{lemma}
\begin{proof}
We take any $f_{\vt}\in\Iset(\vt)$.
If $f_{\vt}=0$, then $0\in\partial\norm{\cdot}\left(f_{\vt}\right)$.
The proof is straightforward by $\vc_{n}:=0$.
Next, we assume that $f_{\vt}\neq0$.
Thus, $f_{\vt}\in\Aset(\vt)$.
If we prove that the (i) holds true,
then Equation \eqref{eq:CharSubdiffNorm} shows that the (ii) holds true. Thus, the proof is completed by showing that an element of $\partial\norm{\cdot}\left(f_{\vt}\right)$
is a linear combination of $\xi^{\ast}_{n1},\xi^{\ast}_{n2},\ldots,\xi^{\ast}_{nN_n}$.
Since $\Aset(\vt)-f_{\vt}=ker\left(\vxi_n^{\ast}\right)$, Equation \eqref{eq:KerLinSpan} shows that
\begin{equation}\label{eq:GenRepr-FiniteRep-1}
\left(\Aset(\vt)-f_{\vt}\right)^{\bot}
=\Span\left\{\xi^{\ast}_{n1},\xi^{\ast}_{n2},\ldots,\xi^{\ast}_{nN_n}\right\}.
\end{equation}
We conclude from Equations \eqref{eq:NormalConeAffine} and \eqref{eq:GenRepr-FiniteRep-1} that
\[
-\partial\chi_{\Aset(\vt)}\left(f_{\vt}\right)
=\Span\left\{\xi^{\ast}_{n1},\xi^{\ast}_{n2},\ldots,\xi^{\ast}_{nN_n}\right\}.
\]
Therefore, the Pshenichnyi-Rockafellar theorem (\cite[Theorem 2.9.1]{Zalinescu2002})
ensures that
\[
\partial\norm{\cdot}\left(f_{\vt}\right)\cap\Span\left\{\xi^{\ast}_{n1},\xi^{\ast}_{n2},\ldots,\xi^{\ast}_{nN_n}\right\}
=
\partial\norm{\cdot}\left(f_{\vt}\right)\cap\left(-\partial\chi_{\Aset(\vt)}\left(f_{\vt}\right)\right)
\neq\emptyset.
\]
\end{proof}

\begin{lemma}\label{Lm:GenRepr-AS}
If $f_0\in\Sset_{n}^{\lambda}\left(\Banach\right)$ and $\vt:=\langle f_0,\vxi_n^{\ast}\rangle$,
then $f_0\in\Iset(\vt)$ and $\Iset(\vt)\subseteq\Sset_{n}^{\lambda}\left(\Banach\right)$.
\end{lemma}
\begin{proof}
We take any $f_{\vt}\in\Iset(\vt)$.
Clearly, $f_0\in\Aset(\vt)$.
Thus, if we prove that $\norm{f_0}\leq\norm{f_{\vt}}$, then $f_0\in\Iset(\vt)$.
Let $h:=f_0-f_{\vt}$.
Since $f_{\vt}\in\Aset(\vt)$, we have $h\in ker\left(\vxi_n^{\ast}\right)$.
Thus, Proposition \ref{Pro:EmpRisk} (a) ensures that
\begin{equation}\label{eq:GenRepr-AS-1}
\risk_n\left(f_{\vt}\right)=\risk_n\left(f_{\vt}+h\right)=\risk_n\left(f_0\right).
\end{equation}
Since $f_0\in\Sset_{n}^{\lambda}\left(\Banach\right)$, we have
\begin{equation}\label{eq:GenRepr-AS-2}
\risk_n\left(f_0\right)+\lambda\regfun\left(\norm{f_0}\right)
\leq\risk_n\left(f_{\vt}\right)+\lambda\regfun\left(\norm{f_{\vt}}\right).
\end{equation}
Subtracting Equations \eqref{eq:GenRepr-AS-1} from \eqref{eq:GenRepr-AS-2}, we have $\regfun\left(\norm{f_0}\right)\leq\regfun\left(\norm{f_{\vt}}\right)$.
Thus, Lemma \ref{Lm:Phi} ensures that $\norm{f_0}\leq\norm{f_{\vt}}$.
This also shows that
\begin{equation}\label{eq:GenRepr-AS-3}
\norm{f_0}=\norm{f_{\vt}}.
\end{equation}
We conclude from Equations \eqref{eq:GenRepr-AS-1} and \eqref{eq:GenRepr-AS-3} that
$\risk_n\left(f_{\vt}\right)+\lambda\regfun\left(\norm{f_{\vt}}\right)
=\risk_n\left(f_0\right)+\lambda\regfun\left(\norm{f_0}\right)$,
hence that $f_{\vt}\in\Sset_{n}^{\lambda}\left(\Banach\right)$, and finally that
$\Iset(\vt)\subseteq\Sset_{n}^{\lambda}\left(\Banach\right)$.
\end{proof}

\begin{proof}[{\bf Proof of Theorem \ref{Thm:GenRepr}}]
Lemma \ref{Lm:GenRepr} first ensures that
$\Sset_{n}^{\lambda}\left(\Banach\right)$ is nonempty.
We take any $f_n^{\lambda}\in\Sset_{n}^{\lambda}\left(\Banach\right)$.
According to Lemma \ref{Lm:GenRepr-AS}, $f_{n}^{\lambda}\in\Iset\left(\vt\right)$,
where $\vt:=\langle f_{n}^{\lambda},\vxi^{\ast}_{n}\rangle$.
Obviously, $\Aset\left(\vt\right)\neq\emptyset$.
Therefore, Lemma \ref{Lm:GenRepr-FiniteRep} also ensures that
the (i) and (ii) hold true.
\end{proof}

Finally, we will complete the proof of Corollary \ref{Cor:GenReprExtPoint} using the techniques of \cite[Theorem 3.1]{BoyerChambolleCastroETC2019}.
\begin{proof}[{\bf Proof of Corollary \ref{Cor:GenReprExtPoint}}]
If $f_{n}^{\lambda}=0$, then we take $\tilde{f}_{n}^{\lambda}:=f_{n}^{\lambda}$.
Next, we assume that $f_{n}^{\lambda}\neq0$.
Let $\vt:=\langle f_{n}^{\lambda},\vxi^{\ast}_{n}\rangle$ and
$\vv:=\vt/\norm{f_n^{\lambda}}$.
Lemma \ref{Lm:GenRepr-AS} ensures that $f_{n}^{\lambda}\in\Iset\left(\vt\right)$.
Thus, $\Iset\left(\vv\right)=\Ball_{\Banach}\cap\Aset\left(\vv\right)\neq\emptyset$.
Let $f_0$ be an extreme point of $\Iset\left(\vv\right)$.
Thus, $\norm{f_n^{\lambda}}\langle f_0,\vxi^{\ast}_{n}\rangle=\langle f_{n}^{\lambda},\vxi^{\ast}_{n}\rangle$.
Moreover, since the codimension of the affine space $\Aset\left(\vv\right)$ is at most $N_n$ and the convex set $\Ball_{\Banach}$ is linear closed and linear bounded,
\cite[Main Theorem]{Dubins1962} ensures that $f_0$ is a convex combination of at most $N_n+1$ extreme points of $\Ball_{\Banach}$.
Thus, $f_0$ is a linear combination of at most $N_n$ extreme points of $\Ball_{\Banach}$.
We take $\tilde{f}_{n}^{\lambda}:=\norm{f_n^{\lambda}}f_0$.
Therefore, $\tilde{f}_{n}^{\lambda}$ is a linear combination of the same extreme points of $\Ball_{\Banach}$ as $f_0$. According to Lemma \ref{Lm:GenRepr-AS}, $\tilde{f}_{n}^{\lambda}\in\Iset(\vv)\subseteq\Sset_{n}^{\lambda}\left(\Banach\right)$.
The proof is complete.
\end{proof}

\begin{remark}
In \cite[Theorem 3.1]{HuangLiuTanYe2020}, we study the generalized representer theorems for the loss functions with a range $[0,\infty]$.
To simplify the proofs of the convergence theorems,
the range of $\loss_n$ is consistently restricted to $[0,\infty)$.
If the range of $\loss_n$ is extended to $[0,\infty]$, then the above proofs of the representer theorems still hold true with additional conditions such as $\loss_n\not\equiv\infty$ and others.
\end{remark}

\subsection{Proofs of Pseudo-approximation Theorems}\label{sec:ProofAppThm}

In this subsection, the $n,\lambda$ are fixed in the proofs of Lemma \ref{Lm:App} and Theorem \ref{Thm:App}.

\begin{proof}[{\bf Proof of Lemma \ref{Lm:App}}]
Since $\Nueral_m$ is weakly* closed,  $\Nueral_m\cap r\Ball_{\Banach}$ is weakly* compact for any $r>0$.
Thus, a slight change in the proof of Lemma \ref{Lm:GenRepr} shows that the conclusions of Lemma \ref{Lm:App} hold true.
\end{proof}

\begin{lemma}\label{Lm:App-bound}
If $\cap_{m\in\NN}\Nueral_m$ is nonempty, then there exists a $r>0$ such that $\norm{f_{nm}^{\lambda}}\leq r$ for all $m\in\NN$.
\end{lemma}
\begin{proof}
We take a $h\in\cap_{m\in\NN}\Nueral_m$. Since $f_{nm}^{\lambda}\in\Sset_{n}^{\lambda}\left(\Nueral_m\right)$,
we have
\[
\risk_n\big(f_{nm}^{\lambda}\big)+\lambda\regfun\big(\norm{f_{nm}^{\lambda}}\big)\leq
\risk_n\left(h\right)+\lambda\regfun\left(\norm{h}\right).
\]
Thus,
\[
\norm{f_{nm}^{\lambda}}\leq\regfun^{-1}\left(\frac{\risk_n(h)+1}{\lambda}+\regfun\left(\norm{h}\right)\right).
\]
\end{proof}

\begin{proof}[{\bf Proof of Theorem \ref{Thm:App}}]
Lemma \ref{Lm:GenRepr} ensures that $\Sset_{n}^{\lambda}\left(\Banach\right)$ is nonempty.
Thus, we take a $g_0\in\Sset_{n}^{\lambda}\left(\Banach\right)$.
The universal approximation guarantees that there exists a sequence $\left(g_{m_k}\right)$ in an increasing order $\left(m_k\right)$ such that
\[
g_{m_k}\in\Nueral_{m_k},\quad\text{for all }k\in\NN,
\]
and
\begin{equation}\label{eq:ThmApp-2}
\lim_{k\to\infty}\norm{g_0-g_{m_k}}=0.
\end{equation}
Since $\loss_n$ is a continuous multi-loss function, $\risk_n$ is weakly* continuous.
Thus, Equation \eqref{eq:ThmApp-2} shows that
\begin{equation}\label{eq:ThmApp-3}
\risk_n\left(g_0\right)+\lambda\regfun\left(\norm{g_0}\right)
=
\lim_{k\to\infty}\risk_n\left(g_{m_k}\right)+\lambda\regfun\left(\norm{g_{m_k}}\right).
\end{equation}

Next, Lemma \ref{Lm:App} ensures that $\Sset_{n}^{\lambda}\left(\Nueral_m\right)$ is nonempty for all $m\in\NN$.
Thus, $\big(f_{nm}^{\lambda}\big)$ is well-defined.
Let $\big(f_{nm_k}^{\lambda}\big)$ be a net with a directed set $\NN$ in its usual order, that is, $k_1\succeq k_2$ if $k_1\geq k_2$. Thus, $\big(f_{nm_k}^{\lambda}\big)$ is a subnet of $\big(f_{nm}^{\lambda}\big)$.
Lemma \ref{Lm:App-bound} ensures that $\big(f_{nm_k}^{\lambda}\big)\subseteq r\Ball_{\Banach}$.
Therefore, the Banach-Alaoglu theorem ensures that there exists a weakly* convergent subnet $\big(f_{nm_{k_{\alpha}}}^{\lambda}\big)$
of $\big(f_{nm_k}^{\lambda}\big)$ to an element $f_0\in r\Ball_{\Banach}$.
Clearly, $\big(f_{nm_{k_{\alpha}}}^{\lambda}\big)$ is also a subnet of $\big(f_{nm}^{\lambda}\big)$ and $\lim_{\alpha}k_{\alpha}=\infty$.
Moreover, we also have
\[
\risk_n\left(f_0\right)+\lambda\regfun\left(\norm{f_0}\right)
\leq
\liminf_{\alpha}\risk_n\big(f_{nm_{k_{\alpha}}}^{\lambda}\big)+\lambda\regfun\big(\norm{f_{nm_{k_{\alpha}}}^{\lambda}}\big).
\]
We will now prove that $f_0\in\Sset_{n}^{\lambda}\left(\Banach\right)$.
Since $g_{m_{k_{\alpha}}}\in\Nueral_{m_{k_{\alpha}}}$, we have
\[
\liminf_{\alpha}\risk_n\big(f_{nm_{k_{\alpha}}}^{\lambda}\big)+\lambda\regfun\big(\norm{f_{nm_{k_{\alpha}}}^{\lambda}}\big)
\leq\liminf_{\alpha}\risk_n\left(g_{m_{k_{\alpha}}}\right)+\lambda\regfun\left(\norm{g_{m_{k_{\alpha}}}}\right).
\]
Thus,
\begin{equation}\label{eq:ThmApp-5-1}
\risk_n\left(f_0\right)+\lambda\regfun\left(\norm{f_0}\right)
\leq
\liminf_{\alpha}\risk_n\left(g_{m_{k_{\alpha}}}\right)+\lambda\regfun\left(\norm{g_{m_{k_{\alpha}}}}\right).
\end{equation}
We conclude from Equations \eqref{eq:ThmApp-3} and \eqref{eq:ThmApp-5-1} that
\[
\risk_n\left(f_0\right)+\lambda\regfun\left(\norm{f_0}\right)
\leq
\risk_n\left(g_0\right)+\lambda\regfun\left(\norm{g_0}\right)
=\inf_{f\in\Banach}\risk_n\left(f\right)+\lambda\regfun\left(\norm{f}\right),
\]
hence that $f_0$ is a minimizer of Optimization \eqref{eq:RERM}.

Since
\[
\risk_n\left(g_0\right)=\lim_{\alpha}\risk_n\left(g_{m_{k_{\alpha}}}\right),
\quad
\regfun\left(\norm{g_0}\right)=\lim_{\alpha}\regfun\left(\norm{g_{m_{k_{\alpha}}}}\right),
\]
and
\begin{equation}\label{eq:ThmApp-9}
\risk_n\left(f_0\right)=\lim_{\alpha}\risk_n\big(f_{nm_{k_{\alpha}}}^{\lambda}\big),
\end{equation}
we have
\[
\lim_{\alpha}\risk_n\left(g_{m_{k_{\alpha}}}\right)-\risk_n\big(f_{nm_{k_{\alpha}}}^{\lambda}\big)
+\lambda\regfun\left(\norm{g_{m_{k_{\alpha}}}}\right)
=
\risk_n\left(g_0\right)-\risk_n\left(f_0\right)+\lambda\regfun\left(\norm{g_0}\right).
\]
It is easy to check that
\[
\risk_n\left(g_0\right)-\risk_n\left(f_0\right)+\lambda\regfun\left(\norm{g_0}\right)
=\lambda\regfun\left(\norm{f_0}\right).
\]
Thus,
\[
\lim_{\alpha}\risk_n\left(g_{m_{k_{\alpha}}}\right)-\risk_n\big(f_{nm_{k_{\alpha}}}^{\lambda}\big)
+\lambda\regfun\left(\norm{g_{m_{k_{\alpha}}}}\right)
=
\lambda\regfun\left(\norm{f_0}\right).
\]
Since $f_{nm_{k_{\alpha}}}\in\Sset_{n}^{\lambda}\left(\Nueral_{m_{k_{\alpha}}}\right)$
and $g_{m_{k_{\alpha}}}\in\Nueral_{m_{k_{\alpha}}}$, we have
\[
\lambda\limsup_{\alpha}\regfun\big(\norm{f_{nm_{k_{\alpha}}}^{\lambda}}\big)\leq
\limsup_{\alpha}\risk_n\left(g_{m_{k_{\alpha}}}\right)-\risk_n\big(f_{nm_{k_{\alpha}}}^{\lambda}\big)
+\lambda\regfun\left(\norm{g_{m_{k_{\alpha}}}}\right).
\]
Thus,
\[
\lambda\limsup_{\alpha}\regfun\big(\norm{f_{nm_{k_{\alpha}}}^{\lambda}}\big)\leq\lambda\regfun\left(\norm{f_0}\right).
\]
It is also easy to check that
\[
\regfun\left(\norm{f_0}\right)\leq
\liminf_{\alpha}\regfun\big(\norm{f_{nm_{k_{\alpha}}}^{\lambda}}\big).
\]
Therefore,
\begin{equation}\label{eq:ThmApp-12}
\regfun\left(\norm{f_0}\right)=\lim_{\alpha}\regfun\big(\norm{f_{nm_{k_{\alpha}}}^{\lambda}}\big).
\end{equation}
Combining Equations \eqref{eq:ThmApp-9} and \eqref{eq:ThmApp-12}, we have
\[
\risk_n\left(f_0\right)+\lambda\regfun\left(\norm{f_0}\right)
=\lim_{\alpha}\risk_n\big(f_{nm_{k_{\alpha}}}^{\lambda}\big)+\lambda\regfun\big(\norm{f_{nm_{k_{\alpha}}}^{\lambda}}\big).
\]
Lemma \ref{Lm:Phi} ensures that
\[
\norm{f_0}=\lim_{\alpha}\norm{f_{nm_{k_{\alpha}}}^{\lambda}}.
\]
Substituting $\big(f_{nm_{k_{\alpha}}}^{\lambda}\big)$ and $f_0$ into $\big(f_{nm_{\alpha}}^{\lambda}\big)$ and $f_{n}^{\lambda}$, the (i) and (ii) hold true.
\end{proof}

\subsection{Proofs of Convergence Theorems}\label{sec:ProofConvThm}

In the beginning of the proof, we assume that Conditions (I) and (II) hold true in this subsection.
This shows that $\loss_n$ is a continuous multi-loss function for any $n\in\NN$.
Thus, Lemma \ref{Lm:GenRepr} ensures that $\Sset_{n}^{\lambda}\left(\Banach\right)$
is nonempty for any $n\in\NN$ and any $\lambda>0$.
In the following analysis, for any given $n,\lambda$, we exclusively choose a fixed $f_{n}^{\lambda}$ from $\Sset_{n}^{\lambda}\left(\Banach\right)$.

\begin{lemma}\label{Lma:L-Cont}
For any $r>0$, $\risk$ is weakly* continuous on $r\Ball_{\Banach}$ and $\risk_n$ uniformly converges to $\risk$
on $r\Ball_{\Banach}$ when $n\to\infty$.
\end{lemma}
\begin{proof}
According to Conditions (I) and (II),
analysis similar to that in the proof of Proposition \ref{Pro:EmpRisk} (d) shows that $\risk\in\Cont\left(r\Ball_{\Banach}\right)$ and
\[
\sup_{f\in r\Ball_{\Banach}}\abs{\risk(f)-\risk_n(f)}\to0,\quad\text{when }n\to\infty.
\]
\end{proof}

\begin{remark}\label{Rm:Gamma-Conv}
According to \cite[Proposition 5.2]{Maso1992}, the uniform convergence guarantees that $\risk_n$ $\Gamma$-converges to $\risk$ on $r\Ball_{\Banach}$ when $n\to\infty$ for any $r>0$.
Some lemmas can be proved using the techniques of $\Gamma$-convergence, such as Lemma \ref{Lma:RERM}.
However, not all convergent results of $f_n^{\lambda}$ can be directly proven by the fundamental theorems of $\Gamma$-convergence.
To circumvent the introduction of multiple concepts in the proof,
the $\Gamma$-convergence is not addressed in this article.
Even though Condition (II$'$) is less stringent compared to Condition (II), we can still prove Lemma \ref{Lma:L-Cont}, when it is substituted for Condition (II).
Roughly speaking, the strong condition of uniform convergence is replaced by the weak condition of pointwise convergence together with the additional condition of linear-functional data and loss functions, which can be checked independent on the original problems.
\end{remark}

Lemmas \ref{Lma:net-conv} and \ref{Lma:inf-limit-ineq} will be often utilized in the following proof.
We first prove Lemma \ref{Lma:net-conv} using Lemma \ref{Lma:L-Cont}.

\begin{lemma}\label{Lma:net-conv}
Let $f_0\in r\Ball_{\Banach}$ and $\left(n_{\alpha},f_{\alpha}\right)\subseteq\NN\times r\Ball_{\Banach}$ be a net
for a $r>0$. Suppose that $\left(n_{\alpha}\right)$ is also a subnet of $\left(n\right)$, where $\left(n\right)$ is a net with the directed set $\NN$ in its usual order.
If $\left(f_{\alpha}\right)$ weakly* converges to $f_0$,
then
\[
\risk\left(f_0\right)
=
\lim_{\alpha}\risk_{n_{\alpha}}\left(f_{\alpha}\right).
\]
\end{lemma}
\begin{proof}
Since $\lim_{\alpha}n_{\alpha}=\lim_{n}n=\infty$,
Lemma \ref{Lma:L-Cont} ensures that
\begin{equation}\label{eq:net-conv-1}
\lim_{\alpha}\abs{\risk\left(f_0\right)-\risk\left(f_{\alpha}\right)}=0,
\end{equation}
and
\begin{equation}\label{eq:net-conv-2}
\lim_{\alpha}\sup_{\beta\in\left(\alpha\right)}
\abs{\risk\left(f_{\beta}\right)-\risk_{n_{\alpha}}\left(f_{\beta}\right)}=0.
\end{equation}
We conclude from Equations \eqref{eq:net-conv-1} and \eqref{eq:net-conv-2} that
\[
\risk\left(f_0\right)=
\lim_{\alpha}\risk\left(f_{\alpha}\right)=
\lim_{\alpha}\risk_{n_{\alpha}}\left(f_{\alpha}\right).
\]
\end{proof}

\begin{lemma}\label{Lma:inf-limit-ineq}
Let $T\in\RR^{\Banach}$ and $\left(T_{\alpha},f_{\alpha}\right)\subseteq\RR^{\Banach}\times\Banach$ be a net. Suppose that
$T_{\alpha}$ converges pointwise to $T$ on $\Banach$,
$f_{\alpha}$ is a minimizer of $T_{\alpha}$ over $\Banach$ for all $\alpha$,
and the limitation of $T_{\alpha}\left(f_{\alpha}\right)$ exists.
If
\[
\inf_{f\in\Banach}T(f)
\leq\lim_{\alpha}T_{\alpha}\left(f_{\alpha}\right),
\]
then
\[
\inf_{f\in\Banach}T(f)
=\lim_{\alpha}T_{\alpha}\left(f_{\alpha}\right).
\]
\end{lemma}
\begin{proof}
We will complete the proof by contradiction.
Let
\begin{equation}\label{eq:RERM-01}
\tau:=\lim_{\alpha}T_{\alpha}\left(f_{\alpha}\right).
\end{equation}
We assume that
\[
\inf_{f\in\Banach}T(f)<\tau.
\]
Thus, there exists a $\rho>0$ such that
\[
\inf_{f\in\Banach}T(f)<\tau-3\rho.
\]
Moreover, there exists a $g_{\rho}\in\Banach$ such that
\[
T\left(g_{\rho}\right)
<\inf_{f\in\Banach}T(f)+\rho.
\]
Therefore,
\[
T\left(g_{\rho}\right)<\tau-2\rho.
\]
This shows that
\[
\lim_{\alpha}T_{\alpha}\left(g_{\rho}\right)
=T\left(g_{\rho}\right)<\tau-2\rho,
\]
and finally that
there exists a $\gamma_1$ such that
\begin{equation}\label{eq:RERM-04}
T_{\alpha}\left(g_{\rho}\right)<\tau-2\rho,
\quad\text{when }\alpha\succeq\gamma_1.
\end{equation}
Next, Equation \eqref{eq:RERM-01} shows that there exists a $\gamma_2$ such that
\begin{equation}\label{eq:RERM-05}
\tau-\rho<
T_{\alpha}\left(f_{\alpha}\right),
\quad\text{when }\alpha\succeq\gamma_2.
\end{equation}
We take a $\beta$ such that $\beta\succeq\gamma_1$ and $\beta\succeq\gamma_2$.
We conclude from Equations \eqref{eq:RERM-04} and \eqref{eq:RERM-05} that
\[
T_{\beta}\left(g_{\rho}\right)
<T_{\beta}\left(f_{\beta}\right)-\rho,
\]
hence that
\[
T_{\beta}\left(g_{\rho}\right)<T_{\beta}\left(f_{\beta}\right)
=\inf_{f\in\Banach}T_{\beta}\left(f\right).
\]
This contradicts the fact of the infimum; therefore we must reject the assumption of the strict inequality.
Thus, the equality holds true.
\end{proof}

The proof of Theorem \ref{Thm:ConvThm} will be divided into three steps including
Lemmas \ref{Lma:RERM}, \ref{Lma:RTRM-lambda}, and \ref{Lma:ConvThm}.
Thus, their notations are consistent.
In the proofs of Lemmas
\ref{Lma:RERM-bound} and \ref{Lma:RERM},
we fix $\lambda>0$.

\begin{lemma}\label{Lma:RERM-bound}
There exists a $N_{\lambda}\in\NN$ such that
$\norm{f_{k}^{\lambda}}\leq r_{\lambda}$ for all $k\geq N_{\lambda}$, where
$r_{\lambda}:=\regfun^{-1}\left(\left(\risk(0)+1\right)/\lambda+\regfun\left(0\right)\right)$.
\end{lemma}
\begin{proof}
Since $\risk_{k}(0)\to\risk(0)$ when $k\to\infty$, there exists a $N_{\lambda}\in\NN$ such that $\risk_{k}(0)\leq\risk(0)+1$ when $k\geq N_{\lambda}$.
We conclude from $f_{k}^{\lambda}\in\Sset_{k}^{\lambda}\left(\Banach\right)$ that
\[
\lambda\regfun\big(\norm{f_{k}^{\lambda}}\big)\leq
\risk_{k}\big(f_{k}^{\lambda}\big)+\lambda\regfun\big(\norm{f_{k}^{\lambda}}\big)
\leq\risk_{k}(0)+\lambda\regfun\left(0\right),
\]
hence that
\[
\norm{f_{k}^{\lambda}}\leq\regfun^{-1}\left(\frac{\risk(0)+1}{\lambda}+\regfun\left(0\right)\right),
\]
and finally that $f_{k}^{\lambda}\in r_{\lambda}\Ball_{\Banach}$ for all $k\geq N_{\lambda}$.
\end{proof}

Let $\big(f_{k}^{\lambda}\big)$ be a net with the directed set
$\left\{k\in\NN:k\geq N_{\lambda}\right\}$ in its usual order,
that is, $k_{1}\succeq k_{2}$ if $k_{1}\geq k_{2}$.
Thus, $\lim_{k}k=\infty$.
Obviously, the directed set of $\big(f_{k}^{\lambda}\big)$ is dependent on the fixed $\lambda$.
Let $\Sset^{\lambda}\left(\Banach\right)$ be a collection of all minimizers of the optimization
\begin{equation}\label{eq:RTRM}
\inf_{f\in\Banach}\risk(f)+\lambda\regfun\left(\norm{f}\right).
\end{equation}

\begin{lemma}\label{Lma:RERM}
There exists a weakly* convergent subnet $\big(f_{k_{\alpha}}^{\lambda}\big)$
of $\big(f_{k}^{\lambda}\big)$ to an element $f^{\lambda}\in\Sset^{\lambda}\left(\Banach\right)$
such that
\[
\text{(i) }\risk\big(f^{\lambda}\big)=\lim_{\alpha}\risk\big(f_{k_{\alpha}}^{\lambda}\big),\quad
\text{(ii) }\norm{f^{\lambda}}=\lim_{\alpha}\norm{f_{k_{\alpha}}^{\lambda}}.
\]
\end{lemma}
\begin{proof}
Lemma \ref{Lma:RERM-bound} ensures that
$\big(f_{k}^{\lambda}\big)\subseteq r_{\lambda}\Ball_{\Banach}$.
Thus, the Banach-Alaoglu theorem ensures that there exists a weakly* convergent subnet $\big(f_{k_{\alpha}}^{\lambda}\big)$
of $\big(f_{k}^{\lambda}\big)$ to an element $f^{\lambda}\in r_{\lambda}\Ball_{\Banach}$.
This shows that $\lim_{\alpha}k_{\alpha}=\lim_{k}k=\infty$.

Next, we prove the (i) and (ii).
Lemma \ref{Lma:L-Cont} ensures that the (i) holds true.
Since $f_{k_{\alpha}}^{\lambda}\in\Sset^{\lambda}_{k_{\alpha}}\left(\Banach\right)$, we have
\[
\risk_{k_{\alpha}}\big(f_{k_{\alpha}}^{\lambda}\big)
+\lambda\regfun\big(\norm{f_{k_{\alpha}}^{\lambda}}\big)
\leq\risk_{k_{\alpha}}
\big(f^{\lambda}\big)+\lambda\regfun\big(\norm{f^{\lambda}}\big).
\]
This shows that
\begin{equation}\label{eq:RERM-2}
\limsup_{\alpha}\risk_{k_{\alpha}}
\big(f_{k_{\alpha}}^{\lambda}\big)+\lambda\regfun\big(\norm{f_{k_{\alpha}}^{\lambda}}\big)
\leq\limsup_{\alpha}\risk_{k_{\alpha}}\big(f^{\lambda}\big)
+\lambda\regfun\big(\norm{f^{\lambda}}\big).
\end{equation}
Substituting $f_0$ and $\left(f_{\alpha}\right)$ into $f^{\lambda}$ and $\big(f^{\lambda}_{k_{\alpha}}\big)$,
Lemma \ref{Lma:net-conv} ensures that
\begin{equation}\label{eq:RERM-3}
\risk\big(f^{\lambda}\big)
=\lim_{\alpha}\risk_{k_{\alpha}}\big(f^{\lambda}_{k_{\alpha}}\big).
\end{equation}
Moreover, Condition (I) shows that
\begin{equation}\label{eq:RERM-4}
\risk\big(f^{\lambda}\big)
=\lim_{\alpha}\risk_{k_{\alpha}}\big(f^{\lambda}\big).
\end{equation}
Subtracting Equations \eqref{eq:RERM-3} and \eqref{eq:RERM-4} from Equation \eqref{eq:RERM-2}, we have
\[
\limsup_{\alpha}\regfun\big(\norm{f_{k_{\alpha}}^{\lambda}}\big)
\leq\regfun\big(\norm{f^{\lambda}}\big).
\]
It is easy to check that
\[
\regfun\big(\norm{f^{\lambda}}\big)
\leq\liminf_{\alpha}\regfun\big(\norm{f_{k_{\alpha}}^{\lambda}}\big).
\]
Thus,
\begin{equation}\label{eq:RERM-7}
\regfun\big(\norm{f^{\lambda}}\big)
=
\lim_{\alpha}\regfun\big(\norm{f_{k_{\alpha}}^{\lambda}}\big).
\end{equation}
According to Lemma \ref{Lm:Phi}, the (ii) holds true.

Finally, we prove that $f^{\lambda}\in\Sset^{\lambda}\left(\Banach\right)$.
Combining Equations \eqref{eq:RERM-3} and \eqref{eq:RERM-7}, we have
\[
\inf_{f\in\Banach}\risk\left(f\right)+\lambda\regfun\left(\norm{f}\right)
\leq\risk\big(f^{\lambda}\big)+\lambda\regfun\big(\norm{f^{\lambda}}\big)
=\lim_{\alpha}\risk_{k_{\alpha}}
\big(f_{k_{\alpha}}^{\lambda}\big)
+\lambda\regfun\big(\norm{f_{k_{\alpha}}^{\lambda}}\big).
\]
Therefore, Lemma \ref{Lma:inf-limit-ineq} ensures that
\[
\inf_{f\in\Banach}\risk\left(f\right)+\lambda\regfun\left(\norm{f}\right)
=
\lim_{\alpha}\risk_{k_{\alpha}}
\big(f_{k_{\alpha}}^{\lambda}\big)
+\lambda\regfun\big(\norm{f_{k_{\alpha}}^{\lambda}}\big)
=
\risk\big(f^{\lambda}\big)+\lambda\regfun\big(\norm{f^{\lambda}}\big),
\]
on substituting $T$, $T_{\alpha}$, and $f_{\alpha}$ into
$\risk+\lambda\regfun\left(\norm{\cdot}\right)$, $\risk_{k_{\alpha}}+\lambda\regfun\left(\norm{\cdot}\right)$, and $f_{k_{\alpha}}^{\lambda}$.
\end{proof}

Let the function
\[
\AffLambda(\lambda):=\inf_{f\in\Banach}\risk\left(f\right)+\lambda\regfun\left(\norm{f}\right),
\quad\text{for }\lambda\in[0,\infty).
\]
Thus, $\AffLambda(0)=\inf_{f\in\Banach}\risk\left(f\right)$.
Since Lemma \ref{Lma:RERM} ensures that $f^{\lambda}\in\Sset^{\lambda}\left(\Banach\right)$ for all $\lambda>0$, we have
\begin{equation}\label{eq:nu-risk}
\AffLambda(\lambda)=\risk\big(f^{\lambda}\big)+\lambda\regfun\big(\norm{f^{\lambda}}\big),
\quad\text{for all }\lambda>0.
\end{equation}

\begin{lemma}\label{Lma:A}
The $\AffLambda$ is concave, continuous, and increasing.
\end{lemma}
\begin{proof}
It is easy to verify that $\AffLambda(\lambda)\to\AffLambda(0)$ when $\lambda\to0$.
Let $\afflambda_{f}(\lambda):=\risk\left(f\right)+\lambda\regfun\left(\norm{f}\right)$ for $\lambda\geq0$.
Thus, $\AffLambda=\inf_{f\in\Banach}\afflambda_{f}$.
Since $\afflambda_{f}$ is affine linear and increasing for any fixed $f\in\Banach$,
$\inf_{f\in\Banach}\afflambda_{f}$ is concave, continuous, and increasing.
\end{proof}

\begin{lemma}\label{Lma:RTRM-bound}
If $\Sset^{0}\left(\Banach\right)$ is nonempty, then
for any $\tilde{f}^0\in\Sset^0\left(\Banach\right)$, we have
\[
\norm{f^{\lambda}}\leq\norm{\tilde{f}^0},\quad\text{for all }\lambda>0.
\]
\end{lemma}
\begin{proof}
We take any $\tilde{f}^0\in\Sset^0\left(\Banach\right)$.
Assume on the contrary that there exist a $\lambda>0$ such that
\[
\norm{\tilde{f}^0}<\norm{f^{\lambda}}.
\]
Since $\risk\big(\tilde{f}^0\big)\leq\risk\big(f^{\lambda}\big)$
and $\regfun\big(\norm{\tilde{f}^0}\big)<\regfun\big(\norm{f^{\lambda}}\big)$,
we have
\[
\risk\big(\tilde{f}^0\big)+\lambda\regfun\big(\norm{\tilde{f}^0}\big)
<\risk\big(f^{\lambda}\big)+\lambda\regfun\big(\norm{f^{\lambda}}\big).
\]
This contradicts the fact that $f^{\lambda}\in\Sset^{\lambda}\left(\Banach\right)$ and
therefore we must reject the assumption.
\end{proof}

Let $\big(f^{\lambda}\big)$ be a net with the directed set $\RR_+$ in its inverse order,
that is, $\lambda_1\succeq\lambda_2$ if $\lambda_1\leq\lambda_2$.
Thus, $\lim_{\lambda}\lambda=0$.

\begin{lemma}\label{Lma:RTRM-lambda}
If $\Sset^{0}\left(\Banach\right)$ is nonempty,
then there exists a weakly* convergent subnet $\big(f^{\lambda_\alpha}\big)$
of $\big(f^{\lambda}\big)$ to an element $f_0\in\Sset^0\left(\Banach\right)$
such that
\[
\text{(i) }\risk\left(f_0\right)=\lim_{\alpha}\risk\big(f^{\lambda_\alpha}\big),\quad
\text{(ii) }\norm{f_0}=\lim_{\alpha}\norm{f^{\lambda_\alpha}}.
\]
\end{lemma}
\begin{proof}
Since $\Sset^0\left(\Banach\right)\neq\emptyset$, Lemma \ref{Lma:RTRM-bound} ensures that
there exists a $r_0>0$ such that $\big(f^{\lambda}\big)\subseteq r_0\Ball_{\Banach}$.
Thus, the Banach-Alaoglu theorem ensures that there exists a weakly* convergent subnet $\big(f^{\lambda_\alpha}\big)$ of $\big(f^{\lambda}\big)$ to an element $f_0\in r_0\Ball_{\Banach}$.
This shows that $\lim_{\alpha}\lambda_{\alpha}=\lim_{\lambda}\lambda=0$.

Next, we prove that $f_0\in\Sset^0\left(\Banach\right)$. Lemma \ref{Lma:L-Cont} ensures that
\begin{equation}\label{eq:RTRM-lambda-1}
\risk\left(f_0\right)=\lim_{\alpha}\risk\big(f^{\lambda_\alpha}\big).
\end{equation}
Combining Equations \eqref{eq:nu-risk} and \eqref{eq:RTRM-lambda-1}, we have
\[
\risk\left(f_0\right)
\leq\liminf_{\alpha}\risk\big(f^{\lambda_\alpha}\big)
+\lambda_{\alpha}\regfun\big(\norm{f^{\lambda_\alpha}}\big)
=\liminf_{\alpha}\AffLambda\left(\lambda_\alpha\right).
\]
Moreover, Lemma \ref{Lma:A} ensures that
\[
\liminf_{\alpha}\AffLambda\left(\lambda_\alpha\right)=
\lim_{\alpha}\AffLambda\left(\lambda_\alpha\right)=\AffLambda\left(0\right)
=\inf_{f\in\Banach}\risk(f).
\]
Thus,
\[
\risk\left(f_0\right)
\leq\inf_{f\in\Banach}\risk(f).
\]
This demonstrates that $\risk\left(f_0\right)=\inf_{f\in\Banach}\risk(f)$.

Finally, we prove the (i) and (ii). Equation \eqref{eq:RTRM-lambda-1} shows that the (i) holds true.
It is easy to check that
\[
\norm{f_0}
\leq\liminf_{\alpha}\norm{f^{\lambda_\alpha}}.
\]
Lemma \ref{Lma:RTRM-bound} ensures that
\[
\limsup_{\alpha}\norm{f^{\lambda_\alpha}}\leq\norm{f_0}.
\]
Thus, the (ii) holds true.
\end{proof}

\begin{lemma}\label{Lma:ConvThm}
If $\Sset^{0}\left(\Banach\right)$ is nonempty,
then
there exists a $f_0\in\Sset^0\left(\Banach\right)$ and
for any $\epsilon>0$,
there exists
a $f_{n_{\epsilon}}^{\lambda_{\epsilon}}\in\Sset_{n_{\epsilon}}^{\lambda_{\epsilon}}\left(\Banach\right)$ for
a $n_{\epsilon}\in\NN$ and a $\lambda_{\epsilon}>0$ such that
\[
\text{(i) }\abs{\risk\left(f_0\right)-\risk\big(f_{n_{\epsilon}}^{\lambda_{\epsilon}}\big)}<\epsilon,\quad
\text{(ii) }\abs{\norm{f_0}-\norm{f_{n_{\epsilon}}^{\lambda_{\epsilon}}}}<\epsilon.
\]
\end{lemma}
\begin{proof}
Lemma \ref{Lma:RTRM-lambda} ensures that
there exists a $f_0\in\Sset^0\left(\Banach\right)$ and
for any $\epsilon>0$,
there exists a $f^{\lambda_{\epsilon}}\in\big(f^{\lambda_{\alpha}}\big)$ given in the proof of Lemma \ref{Lma:RTRM-lambda}, that is, $f^{\lambda_{\epsilon}}\in\Sset^{\lambda_{\epsilon}}\left(\Banach\right)$ for a $\lambda_{\epsilon}>0$, such that
\begin{equation}\label{eq:mainpf-1}
\abs{\risk\left(f_0\right)-\risk\big(f^{\lambda_{\epsilon}}\big)}<\frac{\epsilon}{2},\quad
\abs{\norm{f_0}-\norm{f^{\lambda_{\epsilon}}}}<\frac{\epsilon}{2}.
\end{equation}
Next, for the fixed $\epsilon$ and $\lambda_{\epsilon}$, Lemma \ref{Lma:RERM} ensures that there exists a $f_{n_{\epsilon}}^{\lambda_{\epsilon}}\in\big(f_{k_{\alpha}}^{\lambda_{\epsilon}}\big)$ given in the proof of Lemma \ref{Lma:RERM}, that is, $f_{n_{\epsilon}}^{\lambda_{\epsilon}}\in\Sset_{n_{\epsilon}}^{\lambda_{\epsilon}}\left(\Banach\right)$ for a $n_{\epsilon}\in\NN$, such that
\begin{equation}\label{eq:mainpf-2}
\abs{\risk\big(f^{\lambda_{\epsilon}}\big)-\risk\big(f_{n_{\epsilon}}^{\lambda_{\epsilon}}\big)}<\frac{\epsilon}{2},
\quad
\abs{\norm{f^{\lambda_{\epsilon}}}-\norm{f_{n_{\epsilon}}^{\lambda_{\epsilon}}}}<\frac{\epsilon}{2}.
\end{equation}
Combining Equations \eqref{eq:mainpf-1} and \eqref{eq:mainpf-2}, we complete the proof.
\end{proof}

\begin{remark}\label{Rm:ProofConvThm}
In Lemma \ref{Lma:ConvThm}, $f_0$ is independent on $\epsilon$, and $n_{\epsilon},\lambda_{\epsilon}$ are dependent on $\epsilon$.
The proof above provides additional details. Specifically, we choose the pair $\left(n_{\epsilon},\lambda_{\epsilon}\right)$ for any $\epsilon>0$ such that $n_{\epsilon_1}\geq n_{\epsilon_2}$, $\lambda_{\epsilon_1}\leq\lambda_{\epsilon_2}$ when $\epsilon_1\leq\epsilon_2$ and
$n_{\epsilon}\to\infty$, $\lambda_{\epsilon}\to0$ when $\epsilon\to0$ in the proof of Theorem \ref{Thm:ConvThm}.
\end{remark}

\begin{lemma}\label{Lma:ConvThm-min}
If $\Sset^{0}\left(\Banach\right)$ is nonempty,
then for any $\tilde{f}^0\in\Sset^0\left(\Banach\right)$, we have
\[
\norm{f_0}\leq\norm{\tilde{f}^0},
\]
where $f_0$ is given in Lemma \ref{Lma:ConvThm}.
\end{lemma}
\begin{proof}
Assume on the contrary that $\norm{\tilde{f}^0}<\norm{f_0}$.
Let $\epsilon:=\norm{f_0}-\norm{\tilde{f}^0}$.
As in the proof of the (ii) of Lemma \ref{Lma:ConvThm}, there exists a $\lambda_{\epsilon}>0$ such that
$\norm{f_0}-\norm{f^{\lambda_{\epsilon}}}<\epsilon$.
Thus, $\norm{\tilde{f}^0}=\norm{f_0}-\epsilon<\norm{f^{\lambda_{\epsilon}}}$.
Since $\tilde{f}^0\in\Sset^0\left(\Banach\right)$, Lemma \ref{Lma:RTRM-bound}
ensures that $\norm{f^{\lambda_{\epsilon}}}\leq\norm{\tilde{f}^0}$.
Therefore, it has a contradiction and we must reject the assumption. The proof is complete.
\end{proof}

We will now prove Theorem \ref{Thm:ConvThm} using Lemma \ref{Lma:ConvThm} and
Lemma \ref{Lma:ConvThm-min}.
Let $\big(f_{n_{\epsilon}}^{\lambda_{\epsilon}}\big)$ be a net
with a directed set $\left\{\epsilon\in\RR_{+}:\epsilon\leq1\right\}$ in its inverse order,
that is, $\epsilon_1\succeq \epsilon_2$ if $\epsilon_1\leq \epsilon_2$.
Thus, $\lim_{\epsilon}\epsilon=0$. As explained in Remark \ref{Rm:ProofConvThm}, $\big(f_{n_{\epsilon}}^{\lambda_{\epsilon}}\big)$ is a subnet of $\big(f_{n}^{\lambda}\big)$
and $\lim_{\epsilon}\left(n_{\epsilon},\lambda_{\epsilon}\right)=\left(\infty,0\right)$.

\begin{proof}[{\bf Proof of Theorem \ref{Thm:ConvThm}}]
The (ii) of Lemma \ref{Lma:ConvThm} ensures that
$\norm{f_{n_{\epsilon}}^{\lambda_{\epsilon}}}\leq\norm{f_0}+1$ for all $\epsilon\leq1$.
Let $r_0:=\norm{f_0}+1$.
Thus, $\big(f_{n_{\epsilon}}^{\lambda_{\epsilon}}\big)\subseteq r_0\Ball_{\Banach}$ and
the Banach-Alaoglu theorem ensures that there exists a weakly* convergent subnet $\big(f_{n_{\epsilon_\alpha}}^{\lambda_{\epsilon_\alpha}}\big)$ of $\big(f_{n_{\epsilon}}^{\lambda_{\epsilon}}\big)$ to an element $f^0\in r_0\Ball_{\Banach}$.
This shows that $\lim_{\alpha}\epsilon_{\alpha}=\lim_{\epsilon}\epsilon=0$ and $\big(f_{n_{\epsilon_\alpha}}^{\lambda_{\epsilon_\alpha}}\big)$ is also a subnet of $\big(f_{n}^{\lambda}\big)$
such that $\lim_{\alpha}n_{\epsilon_\alpha}=\infty$ and $\lim_{\alpha}\lambda_{\epsilon_\alpha}=0$.
The proof will be complete for $\big(f_{n_{\alpha}}^{\lambda_{\alpha}}\big)$ when substituted into
$\big(f_{n_{\epsilon_\alpha}}^{\lambda_{\epsilon_\alpha}}\big)$.

Next, we prove that $f^{0}\in\Sset^0\left(\Banach\right)$.
We take any $\delta>0$.
Lemma \ref{Lma:L-Cont} ensures that
there exists a $\gamma_1$ such that
\[
\abs{\risk\big(f^0\big)-\risk\big(f_{n_{\epsilon_\alpha}}^{\lambda_{\epsilon_\alpha}}\big)}
\leq\frac{\delta}{2},\quad\text{when }\alpha\succeq\gamma_1.
\]
Moreover,
Lemma \ref{Lma:ConvThm} ensures that
there exists a $\gamma_2$ such that
\[
\abs{\risk\left(f_0\right)-\risk\big(f_{n_{\epsilon_\alpha}}^{\lambda_{\epsilon_\alpha}}\big)}
\leq\epsilon_{\alpha}\leq\frac{\delta}{2},\quad\text{when }\alpha\succeq\gamma_2.
\]
Let $\gamma_3$ such that $\gamma_3\succeq\gamma_1$ and $\gamma_3\succeq\gamma_2$.
Thus,
\[
\abs{\risk\big(f^0\big)-\risk\left(f_0\right)}
\leq\abs{\risk\big(f^0\big)-\risk
\big(
f_{n_{\epsilon_{\gamma_3}}}^{\lambda_{\epsilon_{\gamma_3}}}
\big)}
+\abs{\risk\big(
f_{n_{\epsilon_{\gamma_3}}}^{\lambda_{\epsilon_{\gamma_3}}}
\big)-\risk\left(f_0\right)}
\leq\frac{\delta}{2}+\frac{\delta}{2}
=\delta.
\]
This demonstrates that
\[
\risk\big(f^0\big)=\risk\left(f_0\right)=\inf_{f\in\Banach}\risk\left(f\right).
\]

We will now prove the (i), (ii), and (iii).
Lemma \ref{Lma:L-Cont} ensures that the (i) holds true.
The (ii) of Lemma \ref{Lma:ConvThm} ensures that
\[
\lim_{\alpha}\abs{
\norm{f_0}
-\norm{f_{n_{\epsilon_\alpha}}^{\lambda_{\epsilon_\alpha}}}
}\leq\lim_{\alpha}\epsilon_{\alpha}=0.
\]
Thus,
\[
\lim_{\alpha}\norm{f_{n_{\epsilon_\alpha}}^{\lambda_{\epsilon_\alpha}}}
=
\norm{f_0}.
\]
It is easy to check that
\[
\norm{f^0}\leq
\liminf_{\alpha}\norm{f_{n_{\epsilon_\alpha}}^{\lambda_{\epsilon_\alpha}}}.
\]
Therefore,
\[
\norm{f^0}\leq\norm{f_0}.
\]
Lemma \ref{Lma:ConvThm-min} also ensures that
\[
\norm{f_0}
\leq\norm{f^0}.
\]
Thus,
\[
\norm{f^0}=\norm{f_0}=\lim_{\alpha}\norm{f_{n_{\epsilon_\alpha}}^{\lambda_{\epsilon_\alpha}}}.
\]
This demonstrates that the (ii) holds true.
Finally, Lemma \ref{Lma:ConvThm-min} ensures that the (iii) holds true.
\end{proof}

Next, we will prove Theorem \ref{Thm:ConvThm-subnet} by the similar argument of Theorem \ref{Thm:ConvThm}.
\begin{proof}[{\bf Proof of Theorem \ref{Thm:ConvThm-subnet}}]
We first prove that $f^0\in\Sset^0\left(\Banach\right)$.
Substituting $f_0$ and $\left(f_{\alpha}\right)$ into $f^0$ and $\big(f_{n_{\alpha}}^{\lambda_{\alpha}}\big)$,
Lemma \ref{Lma:net-conv} ensures that
\begin{equation}\label{eq:main-1}
\risk\big(f^0\big)=
\lim_{\alpha}\risk_{n_{\alpha}}\big(f_{n_{\alpha}}^{\lambda_{\alpha}}\big).
\end{equation}
Since $\big(f_{n_{\alpha}}^{\lambda_{\alpha}}\big)$ is bounded,
there exists a $r_0>0$ such that $\big(f_{n_{\alpha}}^{\lambda_{\alpha}}\big)\subseteq r_0\Ball_{\Banach}$.
Thus, $\lim_{\alpha}\left(n_{\alpha},\lambda_{\alpha}\right)=\lim_{(n,\lambda)}(n,\lambda)=(\infty,0)$.
It is easy to check that
\[
\lim_{\alpha}\lambda_{\alpha}\regfun\big(\norm{f_{n_{\alpha}}^{\lambda_{\alpha}}}\big)
\leq\lim_{\alpha}\lambda_{\alpha}\regfun\left(r_0\right).
\]
Thus,
\begin{equation}\label{eq:main-2}
\lim_{\alpha}\lambda_{\alpha}\regfun\big(\norm{f_{n_{\alpha}}^{\lambda_{\alpha}}}\big)=0.
\end{equation}
Combining Equations \eqref{eq:main-1} and \eqref{eq:main-2}, we have
\[
\inf_{f\in\Banach}\risk(f)\leq
\risk\big(f^0\big)
=\lim_{\alpha}\risk_{n_{\alpha}}\big(f_{n_{\alpha}}^{\lambda_{\alpha}}\big)
+\lambda_{\alpha}\regfun\big(\norm{f_{n_{\alpha}}^{\lambda_{\alpha}}}\big).
\]
Therefore, Lemma \ref{Lma:inf-limit-ineq} ensures that
\[
\inf_{f\in\Banach}\risk(f)=
\lim_{\alpha}\risk_{n_{\alpha}}\big(f_{n_{\alpha}}^{\lambda_{\alpha}}\big)
+\lambda_{\alpha}\regfun\big(\norm{f_{n_{\alpha}}^{\lambda_{\alpha}}}\big)
=\risk\big(f^0\big),
\]
on substituting $T$, $T_{\alpha}$, and $f_{\alpha}$ into $\risk$, $\risk_{n_{\alpha}}+\lambda_{\alpha}\regfun\left(\norm{\cdot}\right)$, and $f_{n_{\alpha}}^{\lambda_{\alpha}}$.

Finally, we prove the (i) and (ii).
Lemma \ref{Lma:L-Cont} ensures that the (i) holds true.
Since $f_{n_{\alpha}}^{\lambda_{\alpha}}\in\Sset_{n_{\alpha}}^{\lambda_{\alpha}}\left(\Banach\right)$, we have
\[
\risk_{n_{\alpha}}\big(f_{n_{\alpha}}^{\lambda_{\alpha}}\big)
+\lambda_{\alpha}\regfun\big(\norm{f_{n_{\alpha}}^{\lambda_{\alpha}}}\big)
\leq
\risk_{n_{\alpha}}\big(f^{0}\big)
+\lambda_{\alpha}\regfun\big(\norm{f^{0}}\big).
\]
Thus,
\[
\regfun\big(\norm{f_{n_{\alpha}}^{\lambda_{\alpha}}}\big)\leq
\frac{\risk_{n_{\alpha}}\big(f^{0}\big)-\risk_{n_{\alpha}}\big(f_{n_{\alpha}}^{\lambda_{\alpha}}\big)}{\lambda_{\alpha}}
+\regfun\big(\norm{f^{0}}\big).
\]
This shows that
\begin{equation}\label{eq:main-3}
\limsup_{\alpha}\regfun\big(\norm{f_{n_{\alpha}}^{\lambda_{\alpha}}}\big)\leq
\limsup_{\alpha}\frac{\risk_{n_{\alpha}}\big(f^{0}\big)-\risk_{n_{\alpha}}\big(f_{n_{\alpha}}^{\lambda_{\alpha}}\big)}{\lambda_{\alpha}}
+\regfun\big(\norm{f^{0}}\big).
\end{equation}
By the additional condition in Equation \eqref{eq:add-cond-weak-net},  Proposition \ref{Pro:EmpRisk} (e) ensures that
\begin{equation}\label{eq:main-4}
\lim_{\alpha}
\frac{\abs{\risk_{n_{\alpha}}\big(f^{0}\big)-\risk_{n_{\alpha}}\big(f_{n_{\alpha}}^{\lambda_{\alpha}}\big)}}{\lambda_{\alpha}}
=0.
\end{equation}
Subtracting Equations \eqref{eq:main-4} from \eqref{eq:main-3}, we have
\[
\limsup_{\alpha}\regfun\big(\norm{f_{n_{\alpha}}^{\lambda_{\alpha}}}\big)
\leq
\regfun\big(\norm{f^0}\big).
\]
It is also easy to check that
\[
\regfun\big(\norm{f^0}\big)\leq
\liminf_{\alpha}\regfun\big(\norm{f_{n_{\alpha}}^{\lambda_{\alpha}}}\big).
\]
This demonstrates that
\[
\regfun\big(\norm{f^0}\big)=
\lim_{\alpha}\regfun\big(\norm{f_{n_{\alpha}}^{\lambda_{\alpha}}}\big).
\]
Thus, the (ii) holds true according to Lemma \ref{Lm:Phi}.
\end{proof}

Finally, we prove Theorem \ref{Thm:ConvThm-lambda} using Theorem \ref{Thm:ConvThm-subnet}.
\begin{proof}[{\bf Proof of Theorem \ref{Thm:ConvThm-lambda}}]
We first prove that $\big(f_{n}^{\lambda_n}\big)$ is bounded.
Since $f_{n}^{\lambda_n}\in\Sset_{n}^{\lambda_n}\left(\Banach\right)$, we have
\begin{equation}\label{eq:ConvThm-lambda-1}
\regfun\big(\norm{f_{n}^{\lambda_n}}\big)
\leq\frac{\risk_n\big(f^0\big)-\risk_n\big(f_{n}^{\lambda_n}\big)}{\lambda_n}
+\regfun\big(\norm{f^0}\big)
\leq
\frac{\risk_n\big(f^0\big)}{\lambda_n}
+\regfun\big(\norm{f^0}\big).
\end{equation}
Combining Equation \eqref{eq:ConvThm-lambda-1} and the additional condition in Equation \eqref{eq:add-cond-strong-seq}, we have
\[
\limsup_{n\to\infty}\regfun\big(\norm{f_{n}^{\lambda_n}}\big)
\leq
\regfun\big(\norm{f^0}\big).
\]
Thus, the strictly increasing of $\regfun$ shows that
\begin{equation}\label{eq:ConvThm-lambda-3}
\limsup_{n\to\infty}\norm{f_{n}^{\lambda_n}}
\leq
\norm{f^0}.
\end{equation}
Therefore, there exists a
$r_0>0$ such that $\big(f_{n}^{\lambda_n}\big)\subseteq r_0\Ball_{\Banach}$.

Next, we prove that $\big(f_{n}^{\lambda_n}\big)$ is a weakly* convergent sequence to $f^0$.
If we prove the statement, then Lemma \ref{Lma:L-Cont} ensures that the (i) holds true.
The Banach-Alaoglu theorem ensures that there exists a weakly* convergent subnet $\big(f^{\lambda_{n_{\alpha}}}_{n_{\alpha}}\big)$ of $\big(f_{n}^{\lambda_n}\big)$ to an element $f_0\in r_0\Ball_{\Banach}$.
It is obvious that $\big(f^{\lambda_{n_{\alpha}}}_{n_{\alpha}}\big)$ is also a subnet of $\big(f_{n}^{\lambda}\big)$.
Thus, Theorem \ref{Thm:ConvThm-subnet} ensures that $f_0\in\Sset^0\left(\Banach\right)=\left\{f^0\right\}$.
This demonstrates that
\[
f^0=f_0=w^{\ast}-\lim_{\alpha}f^{\lambda_{n_{\alpha}}}_{n_{\alpha}}.
\]
Therefore, Lemma \ref{Lma:L-Cont} ensures that
\[
\risk\big(f^0\big)=\lim_{\alpha}\risk\big(f^{\lambda_{n_{\alpha}}}_{n_{\alpha}}\big).
\]
It is easy to check that
\[
\norm{f^0}
\leq\liminf_{{\alpha}}\norm{f_{n_{\alpha}}^{\lambda_{n_\alpha}}}.
\]
Since $\lim_{\alpha}n_{\alpha}=\lim_nn=\infty$,
Equation \eqref{eq:ConvThm-lambda-3} also shows that
\[
\limsup_{\alpha}\norm{f_{n_{\alpha}}^{\lambda_{n_\alpha}}}
\leq
\norm{f^0}.
\]
Thus,
\[
\norm{f^0}=\lim_{\alpha}\norm{f_{n_{\alpha}}^{\lambda_{n_\alpha}}}.
\]
Applying the same argument to all subnets of $\big(f_{n}^{\lambda_n}\big)$, we assert that
\[
f^0=w^{\ast}-\lim_{n\to\infty}f_{n}^{\lambda_n},\quad
\norm{f^0}=\lim_{n\to\infty}\norm{f_{n}^{\lambda_n}}.
\]
\end{proof}

\section{Examples of Regularized Learning}\label{sec:Exa}

In the beginning, we study an example of noisy data.
\begin{example}\label{exa:Gaussian}
{\rm
Let $\Omega:=[0,1]$ and a Gaussian kernel
\[
K(x,z):=\exp\big(-\theta^2\abs{x-z}^2\big),\quad\text{for }x,z\in\Omega,
\]
where the shape parameter $\theta>0$. As in \cite[Section 4.4]{XuYe2019},
the RKBS $\RKBS_K^{1}(\Omega)$ has a predual space $\left(\RKBS_K^{1}(\Omega)\right)_{\ast}\cong\RKBS_K^{\infty}(\Omega)$ and $\RKBS_K^{1}(\Omega)$ is embedded into the RKHS $\Hilbert_K(\Omega)$.
As in \cite[Example 5.8]{FasshauerYe2011Dist}, $\Hilbert_K(\Omega)$ is embedded into the Sobolev space $\Hilbert^j(\Omega)$ of any order $j$. Thus, the Sobolev embedding theorem ensures that $\RKBS_K^{1}(\Omega)$ is embedded into $\Cont^{0,1}(\Omega)$.

We consider an expected risk function
\[
\risk(f):=\int_{\Omega}\abs{f(x)-f^0(x)}\omega(x)\ud x,\quad\text{for }f\in\RKBS_K^{1}(\Omega),
\]
where the unknown function $f^0\in\RKBS_K^{1}(\Omega)$ and the positive weight $\omega\in\Cont(\Omega)$.
This shows that $\risk\left(f^0\right)=0$ and $f^0$ is the unique minimizer of $\risk$ over $\RKBS_K^{1}(\Omega)$, that is, $\Sset^0\left(\RKBS_K^{1}(\Omega)\right)=\left\{f^0\right\}$.

Next, we have the points $x_{nk}\in((k-1)/n,k/n)$ and the noise values $y_{nk}\in\RR$ such that $\abs{f^0(x_{nk})-y_{nk}}\leq\zeta_{nk}$ for $k\in\NN_n$ and $n\in\NN$, where $\zeta_{nk}>0$ and $\max_{k\in\NN_n}\zeta_{nk}\to0$ when $n\to\infty$.
Thus, there exists a $C>0$ such that $\sup_{x\in\Omega}\abs{f^0(x)}+\sup_{n\in\NN}\max_{k\in\NN_n}\zeta_{nk}\leq C$.
This demonstrates that $f^0\left(x_{nk}\right),y_{nk}\in[-C,C]$.
By the reproducing property of $\RKBS_K^{1}(\Omega)$, the classical noisy data
\[
\left(x_{n1},y_{n1}\right),\left(x_{n2},y_{n2}\right),\ldots,\left(x_{nn},y_{nn}\right)\in\Omega\times[-C,C],\quad\text{for all }n\in\NN,
\]
are equivalently transferred to the data
\[
\left(\delta_{x_{n1}},y_{n1}\right),\left(\delta_{x_{n2}},y_{n2}\right),\ldots,\left(\delta_{x_{nn}},y_{nn}\right)
\in\RKBS_K^{\infty}(\Omega)\times[-C,C],\quad\text{for all }n\in\NN.
\]
Let
\[
\vxi_n:=\left(\delta_{x_{n1}},\delta_{x_{n2}},\cdots,\delta_{x_{nn}}\right),
\quad
\vy_n:=\left(y_{n1},y_{n2},\cdots,y_{nn}\right),
\quad\text{for all }n\in\NN.
\]
Thus, $\Funset_{\Data}=\left\{\delta_{x_{nk_n}}:k_n\in\NN_n,~n\in\NN\right\}$
and Lemma \ref{Lm:weak-equicont} ensures that
$\Funset_D$ is relatively compact in $\RKBS_K^{\infty}(\Omega)$.
Let a multi-loss function
\[
\loss_n\left(\vxi^{\ast},\vy,\vt\right):=\frac{1}{n}\sum_{k=1}^n
\closs\left(\xi_k^{\ast},y_{k},t_k\right),
\text{ for }\vxi^{\ast}\in\big(\RKBS_K^{\infty}(\Omega)\big)^{n},
~\vy\in[-C,C]^n,~\vt\in\RR^n,
\]
where $n\in\NN$ and absolute loss
\[
\closs\left(\xi^{\ast},y,t\right):=\left\{
\begin{array}{cl}
\abs{t-y}\omega(x),&\text{if }\xi^{\ast}=\delta_{x}\text{ for }x\in\Omega\text{ and }y\in[-C,C],\\
0,&\text{otherwise},
\end{array}
\right.
\]
for $\xi^{\ast}\in\RKBS_K^{\infty}(\Omega),~y\in[-C,C],~t\in\RR$.
Since $\omega$ is bounded on $\Omega$,
$\closs$ is a local Lipschitz continuous loss function.
Thus, Lemma \ref{Lm:loss} (b) ensures that $\Loss$ is uniformly local Lipschitz continuous.
This demonstrates that Condition (II) holds true.

Moreover, the empirical risk function is written as
\[
\risk_n(f)=
\frac{1}{n}\sum_{k=1}^n
\closs\left(\delta_{x_{nk}},y_{nk},\langle f,\delta_{x_{nk}}\rangle\right)
=
\frac{1}{n}\sum_{k=1}^{n}\abs{f\left(x_{nk}\right)-y_{nk}}\omega\left(x_{nk}\right),
\]
for $f\in\RKBS_K^{1}(\Omega)$.
Since $\abs{f-f^0}\omega$ is continuous and $\omega$ is bounded, $\risk_n$ converges pointwise to $\risk$ when $n\to\infty$. This demonstrates that Condition (I) holds true.

Let $\regfun(r):=r$ for $r\in[0,\infty)$.
Since the extreme point of $\Ball_{\RKBS_K^{1}(\Omega)}$ has the formula as $\sqrt{\vartheta}\varphi$, Corollary \ref{Cor:GenReprExtPoint} ensures that there exists an approximate solution
$f_n^{\lambda}\in\Sset_{n}^{\lambda}\left(\RKBS_K^{1}(\Omega)\right)$
such that $f_n^{\lambda}$ is a linear combination of
\[
\sqrt{\vartheta_1}\varphi_1,\sqrt{\vartheta_2}\varphi_2,\ldots,\sqrt{\vartheta_{M_n}}\varphi_{M_n},
\]
where $M_n\leq N_n$
and $\left(\vartheta,\varphi\right)$ is a pair of eigenvalue and eigenfunction of $K$.

Finally, we choose a decrease sequence $\left(\lambda_n\right)$ to $0$ such that
\[
\lim_{n\to\infty}\frac{\max_{k\in\NN_{n}}\zeta_{nk}}{\lambda_n}=0.
\]
Since
\[
\risk_n\big(f^0\big)=\frac{1}{n}\sum_{k=1}^{n}\abs{f^0\left(x_{nk}\right)-y_{nk}}\omega\left(x_{nk}\right)
\leq2\sup_{x\in\Omega}\omega(x)\max_{k\in\NN_{n}}\zeta_{nk},
\]
we have
\[
\lim_{n\to\infty}\frac{\risk_n\left(f^0\right)}{\lambda_n}
\leq2\sup_{x\in\Omega}\omega(x)\lim_{n\to\infty}
\frac{\max_{k\in\NN_{n}}\zeta_{nk}}{\lambda_n}=0.
\]
Thus, Theorem \ref{Thm:ConvThm-lambda} ensures that $\big(f_{n}^{\lambda_{n}}\big)$ is a weakly* convergent bounded sequence to $f^0$. This demonstrates that $f_{n}^{\lambda_{n}}$ converges pointwise to $f^0$ when $n\to\infty$.
This result is consistent with the variational characterization of Tikhonov regularization.
}
\end{example}

In this article, we focus on deterministic convergence.
Actually, the pointwise convergence as shown in Equation \eqref{eq:PointConv} can be extended to stochastic convergence, such as
convergence almost surely and convergence in probability.
In the same manner, we can verify the conclusions of the convergence theorems for stochastic data.
Now, let us study an example of a binary classification for stochastic data.
\begin{example}\label{exa:BinaryClass}
{\rm
Let $\Omega:=[0,1]^2$ and
a min kernel
\[
K(x,z):=\min\left\{v_1,w_1\right\}\min\left\{v_2,w_2\right\},
\text{ for }x=\left(v_1,v_2\right),z=\left(w_1,w_2\right)\in\Omega.
\]
As stated in \cite[Section 4.3]{XuYe2019}, the RKBS $\RKBS_K^p(\Omega)$ is a reflexive Banach space
and its dual space is isometrically isomorphic to the RKBS $\RKBS_K^q(\Omega)$,
where $1<p\leq2$ and $q=p/(p-1)$.
As shown in \cite[Example 5.1]{FasshauerYe2011DiffBound},
the RKHS $\Hilbert_K(\Omega)$
is equivalent to the Sobolev space $\Hilbert^{1,1}_{mix}(\Omega)$ of order $1,1$ with boundary conditions.
Since $\RKBS_K^p(\Omega)$ is embedded into $\Hilbert_K(\Omega)$,
the Sobolev embedding theorem ensures that
$\RKBS_K^p(\Omega)$ is embedded into $\Cont^{0,1/2}(\Omega)$.

Given a probability distribution $\PP$ on $\Omega\times\{\pm1\}$,
we study an expected risk function
\[
\risk(f):=\int_{\Omega\times\{\pm1\}}\max\left\{0,1-yf(x)\right\}\ud\PP,
\quad\text{for }f\in\RKBS_K^p(\Omega).
\]
Thus, the binary classifier is constructed by a minimization of $\risk$ over $\RKBS_K^p(\Omega)$.

Next, we have the stochastic data $\left(x_n,y_n\right)\overset{\PP}{\sim}i.i.d.\left(x,y\right)$ for all $n\in\NN$.
By the reproducing property of $\RKBS_K^{p}(\Omega)$, the classical data $\left(x_{n},y_{n}\right)$
are equivalently transferred to the data
$\left(\delta_{x_{n}},y_{n}\right)$.
Let
\[
\vxi_n:=\left(\delta_{x_{1}},\delta_{x_{2}},\cdots,\delta_{x_{n}}\right),
\quad
\vy_n:=\left(y_{1},y_{2},\cdots,y_{n}\right),
\quad\text{for all }n\in\NN.
\]
Thus, $\Funset_{\Data}=\left\{\delta_{x_n}:n\in\NN\right\}$.
Lemma \ref{Lm:weak-equicont} ensures that
$\Funset_{\Data}$ is relatively compact in $\RKBS_K^{q}(\Omega)$.
Let a multi-loss function
\[
\loss_n\left(\vxi^{\ast},\vy,\vt\right):=\frac{1}{n}\sum_{k=1}^n
\closs\left(\xi_k^{\ast},y_{k},t_k\right),
\text{ for }\vxi^{\ast}\in\big(\RKBS_K^{q}(\Omega)\big)^{n},
~\vy\in\{\pm1\}^n,~\vt\in\RR^n,
\]
where $n\in\NN$ and
hinge loss
\[
\closs\left(\xi^{\ast},y,t\right):=\left\{
\begin{array}{cl}
\max\left\{0,1-yt\right\},&\text{if }\xi^{\ast}=\delta_{x}\text{ for }x\in\Omega\text{ and }y\in\left\{\pm1\right\},\\
0,&\text{otherwise},
\end{array}
\right.
\]
for $\xi^{\ast}\in\RKBS_K^{q}(\Omega),~y\in\{\pm1\},~t\in\RR$.
Since $\closs$ is a local Lipschitz continuous loss function,
Lemma \ref{Lm:loss} (b) ensures that $\Loss$ is uniformly local Lipschitz continuous.
This demonstrates that Condition (II) holds true.

Moreover, since $L\left(\delta_{x_{n}},y_n,\langle f,\delta_{x_{n}}\rangle\right)\overset{\PP}{\sim}i.i.d.\max\left\{0,1-yf(x)\right\}$,
the strong law of large numbers shows that
\[
\risk(f)=
\lim_{n\to\infty}
\frac{1}{n}\sum_{k=1}^{n}L\left(\delta_{x_{k}},y_k,\langle f,\delta_{x_{k}}\rangle\right)
=\lim_{n\to\infty}\risk_n(f)
\text{ almost surely},
\]
for all $f\in\RKBS_K^{p}(\Omega)$.
Thus, the separability of $\RKBS_K^{p}(\Omega)$ and the Lipschitz condition of hinge loss demonstrate that Condition (I) holds true almost surely.

Let $\regfun(r):=r^p$ for $r\in[0,\infty)$.
If $p=2l/(2l-1)$ for a $l\in\NN$ and $f_n^{\lambda}\neq0$, then Theorem \ref{Thm:GenRepr} ensures that
there exists an unique approximate solution $f_n^{\lambda}\in\Sset_{n}^{\lambda}\left(\RKBS_K^{p}(\Omega)\right)$
such that $\partial\norm{\cdot}\big(f_n^{\lambda}\big)$ is a singleton including the unique element that is equivalent to a linear combination
of $K\left(x_1,\cdot\right),K\left(x_2,\cdot\right),\ldots,K\left(x_n,\cdot\right)$
Thus, $f_n^{\lambda}$ is a linear combination of multi-kernel basis
\[
K_{2l}\left(\cdot,x_{i_1},x_{i_2},\cdots,x_{i_{2l-1}}\right),
\quad\text{for }i_1,i_2,\ldots,i_{2l-1}\in\NN_n,
\]
constructed by eigenvalues and eigenfunctions of $K$, as shown in the proof of \cite[Theorem 5.10]{XuYe2019}.

In the same manner as stated in Section \ref{sec:ProofConvThm}, we exclude the null set of the unconvergent $\risk_n$
to prove that Theorems \ref{Thm:ConvThm}, \ref{Thm:ConvThm-subnet}, and \ref{Thm:ConvThm-lambda} still hold true almost surely.
Since $\RKBS_K^{q}(\Omega)$ is separable and complete, the weakly* convergent bounded subnet of $\big(f_{n}^{\lambda}\big)$ in Theorem \ref{Thm:ConvThm-subnet} can be interchanged to a weakly* convergent sequence $\big(f_{n_{j}}^{\lambda_{j}}\big)$.
Since $\Span\left\{\delta_x:x\in\Omega\right\}$ is dense in $\RKBS_K^{q}(\Omega)$,
$\left(f_j\right)$ weakly* converges to $f_0$ when $j\to\infty$
if and only if $\left(f_j\right)$ converges pointwise to $f_0$ when $j\to\infty$,
where $f_0\in\RKBS_K^{p}(\Omega)$ and $\left(f_j\right)$ is a sequence of $\RKBS_K^{p}(\Omega)$ with $sup_{j\in\NN}\norm{f_j}<\infty$.
In practical applications, if a stochastic bounded sequence $\big(f_{n_{j}}^{\lambda_{j}}\big)$ converges pointwise to a deterministic $f^0\in\RKBS_K^{p}(\Omega)$ when $j\to\infty$ almost surely, then Theorem \ref{Thm:ConvThm-subnet} ensures that $f^0\in\Sset^0\left(\RKBS_K^{p}(\Omega)\right)$.
Thus, the binary classifier is constructed by
$\sign\left(f^0(x)\right)$ or $\sign\big(f_{n_{j}}^{\lambda_{j}}(x)\big)$ approximately for $x\in\Omega$.
}
\end{example}

\begin{remark}
For binary classification, we can still achieve similar results as shown in Example \ref{exa:BinaryClass} by replacing the hinge loss with various other loss functions, such as truncated least squares loss, logistic loss, and even nonconvex loss.
According to Zhang's inequality in \cite[Theorem 2.31]{SteinwartChristmann2008}, the binary classifier can be written as the Bayes classifier $\sign\left(2\omega(x)-1\right)$,
where $\omega(x):=\PP(y=1\mid x)$ for $x\in\Omega$.
Unfortunately, $\omega$ is usually unknown.
Since $\Cont_0(\Rd)$ is a predual space of $\Measure(\Rd)$, we will utilize total variations to construct approximate classifiers using nonconvex and nonsmooth loss functions, when $2\omega-1\in\Measure(\Rd)$ in our forthcoming paper.
\end{remark}

Next, we study a specific example involving two categories of linear-functional data induced by a Poisson equation.
\begin{example}\label{exa:Poisson}
{\rm
Let $\Omega:=[0,1]^2$ and a Sobolev kernel
\[
K(x,z):=\theta^{j-1}\norm{x-z}^{j-1}_2\Bessel_{1-j}\left(\theta\norm{x-z}_2\right),
\quad\text{for }x,z\in\Omega,
\]
where $\theta>0$, $j\geq4$, and $\Bessel_{1-j}$ is the modified Bessel function of the second kind of order $1-j$.
The RKHS $\Hilbert_K(\Omega)$ is a reflexive Banach space.
As shown in \cite[Example 5.7]{FasshauerYe2011Dist}, $\Hilbert_K(\Omega)$ is equivalent to the Sobolev space $\Hilbert^j(\Omega)$ of order $j$.
By the Sobolev embedding theorem,
$\Hilbert_K(\Omega)$ is embedded into $\Cont^{2,1}(\Omega)$ and $\Cont^{0,1}(\partial\Omega)$, respectively.

We study a Poisson equation with Dirichlet boundary, that is,
\[
\Delta u=h\text{ in }\Omega,\quad
u=g\text{ on }\partial\Omega,
\]
where $h\in\Cont(\Omega)$ and $g\in\Cont(\partial\Omega)$ are given such that
the Poisson equation exists the unique solution $f^0\in\Hilbert^j(\Omega)\cong\Hilbert_K(\Omega)$.
Let an expected risk function
\[
\risk(f):=\int_{\Omega}\abs{\Delta f(x)-h(x)}^2\ud x
+\frac{1}{4}\int_{\partial\Omega}\abs{f(z)-g(z)}^2\ud S,
\quad\text{for }f\in\Hilbert_K(\Omega).
\]
Thus, $\risk\left(f^0\right)=0$ and $\Sset^0\left(\Hilbert_K(\Omega)\right)=\left\{f^0\right\}$.

Next, we have two categories of linear-functional data induced by the Poisson equation, that is,
\[
\left(\delta_{x_{n1}}\circ\Delta,v_{n1}\right),\left(\delta_{x_{n2}}\circ\Delta,v_{n2}\right),
\ldots,\left(\delta_{x_{nn^2}}\circ\Delta,v_{nn^2}\right),
\]
and
\[
\left(\delta_{z_{n1}},b_{n1}\right),\left(\delta_{z_{n2}},b_{n2}\right),
\ldots,\left(\delta_{z_{nn}},b_{nn}\right),
\]
where $v_{nk}:=h\left(x_{nk}\right)$,
$x_{nk}$ are Halton points in $\Omega$, $b_{nl}:=g\left(z_{nl}\right)$,
and $z_{nl}$ are uniform grid points on $\partial\Omega$.
Let
\[
\vxi_n^{\ast}:=\left(
\delta_{x_{n1}}\circ\Delta,\delta_{x_{n2}}\circ\Delta,
\cdots,\delta_{x_{nn^2}}\circ\Delta,
\delta_{z_{n1}},\delta_{z_{n2}},\cdots,\delta_{z_{nn}}
\right),
\]
and
\[
\vy_n:=\left(v_{n1},v_{n2},\cdots,v_{nn^2},b_{n1},b_{n2},\cdots,b_{nn}\right),
\]
for all $n\in\NN$.
Thus, $\Funset_{\Data}=\left\{\delta_{x_{nk_n}}\circ\Delta,\delta_{z_{nl_n}}:k_n\in\NN_{n^2},l_n\in\NN_n,n\in\NN\right\}$.
Since Halton points are deterministic, the linear-functional data are deterministic.
If we choose random Sobol points, then the linear-functional data are stochastic.
Analysis similar to that in the proof of Lemma \ref{Lm:weak-equicont} shows that $\Funset_{\Data}$ is relatively compact in $\Hilbert_K(\Omega)$.
Let a multi-loss function
\[
\loss_n\left(\vxi^{\ast},\vy,\vt\right):=
\frac{1}{n^2}\sum_{k=1}^{n^2}L^{1}\left(\xi_{k}^{\ast},y_{k},t_k\right)
+\frac{1}{n}\sum_{l=1}^nL^{2}\left(\xi_{n^2+l}^{\ast},y_{n^2+l},t_{n^2+l}\right),
\]
for $\vxi^{\ast}\in\left(\Hilbert_K(\Omega)\right)^{n^2+n},~\vy\in\RR^{n^2+n},~\vt\in\RR^{n^2+n}$, where
the square losses
\[
L^{1}\left(\xi^{\ast},y,t\right):=\left\{
\begin{array}{cl}
\abs{t-y}^2,&\text{if }\xi^{\ast}=\delta_{x}\circ\Delta\text{ for }x\in\Omega\text{ and }y\in range(h),\\
0,&\text{otherwise},
\end{array}
\right.
\]
and
\[
L^{2}\left(\xi^{\ast},y,t\right):=\left\{
\begin{array}{cl}
\abs{t-y}^2,&\text{if }\xi^{\ast}=\delta_{z}\text{ for }z\in\partial\Omega\text{ and }y\in range(g),\\
0,&\text{otherwise},
\end{array}
\right.
\]
for $\xi^{\ast}\in\Hilbert_K(\Omega),~y\in\RR,~t\in\RR$.
Since $h,g$ are bounded on $\Omega$ and $\partial\Omega$, respectively, $L^1,L^2$ are both local Lipschitz continuous loss functions.
In the same manner as Lemma \ref{Lm:loss} (b), $\Loss$ is uniformly local Lipschitz continuous.
This demonstrates that Condition (II) holds true.

Since
\[
\risk_n(f)=\frac{1}{n^2}\sum_{k=1}^{n^2}\abs{\Delta f\left(x_{nk}\right)-h\left(x_{nk}\right)}^2
+\frac{1}{n}\sum_{l=1}^n\abs{f\left(z_{nl}\right)-g\left(z_{nl}\right)}^2,
\]
for $f\in\Hilbert_K(\Omega)$,
the Koksma–Hlawka inequality in quasi-Monte Carlo methods shows that
$\risk_n$ converges pointwise to $\risk$ when $n\to\infty$.
This demonstrates that Condition (I) holds true.

Let $\regfun(r):=r^2$ for $r\in[0,\infty)$.
The reproducing property of $\Hilbert_K(\Omega)$ shows that $\delta_{x}\circ\Delta\cong\Delta_{x}K(x,\cdot)$
and $\delta_{z}\cong K(z,\cdot)$ for $x\in\Omega$ and $z\in\partial\Omega$.
Thus, Theorem \ref{Thm:GenRepr} ensures that there exists the unique approximate solution
$f_n^{\lambda}\in\Sset_{n}^{\lambda}\left(\Hilbert_K(\Omega)\right)$
such that $f_n^{\lambda}$ is a linear combination of
\[
\Delta_{x}K\left(x_{n1},\cdot\right),\Delta_{x}K\left(x_{n2},\cdot\right),\ldots,\Delta_{x}K\left(x_{nn^2},\cdot\right),
\]
and
\[
K\left(z_{n1},\cdot\right),K\left(z_{n2},\cdot\right),\ldots,K\left(z_{nn},\cdot\right).
\]

Finally,
we can verify that $\risk_n\left(f^0\right)=0$ for all $n\in\NN$.
Therefore, for any decrease sequence $\left(\lambda_n\right)$ to $0$,
Theorem \ref{Thm:ConvThm-lambda} ensures that $f_{n}^{\lambda_{n}}$ weakly* converges to $f^0$
and $\norm{f_{n}^{\lambda_{n}}}$ converges to $\norm{f^0}$ when $n\to\infty$.
Since $\Hilbert_K(\Omega)$ is a Radon-Riesz space,
we have $\norm{f^0-f_{n}^{\lambda_{n}}}\to0$ when $n\to\infty$.
This demonstrates that
\[
\lim_{n\to\infty}\sup_{x\in\Omega}\abs{\partial^{\vartheta}f^0(x)-\partial^{\vartheta}f_{n}^{\lambda_{n}}(x)}=0,
\quad\text{for any }\abs{\vartheta}\leq2.
\]
}
\end{example}

In Examples \ref{exa:Gaussian}, \ref{exa:BinaryClass}, and \ref{exa:Poisson}, we utilize the reproducing kernels to directly solve the approximate solutions according to the representer theorems.
Now, let us study a popular example of artificial neural networks.
Under the universal approximation of two-layer neural networks, we will use sigmoid functions to approximately compute the approximate solutions according to the pseudo-approximation theorems.
\begin{example}\label{exa:neural-network}
{\rm
Let $\Omega:=[-1,1]^d$. It is obvious that $\Leb_{q}(\Omega)$ is a predual space of $\Leb_{p}(\Omega)$, where $p,q>1$ and $1/p+1/q=1$.
Let $\Nueral_m$ be a collection of all two-layer neural networks with $m$ coefficients, that is,
\[
f_m(x)=\vW_2\vsigma\left(\vW_1x+\vb_1\right)+\vb_2,\quad\text{for }x\in\Omega,
\]
where $\vsigma$ is a continuous sigmoid function, $\vW_1,\vW_2$ are weight matrixes, and $\vb_1,\vb_2$ are bias vectors as stated in \cite[Section 6.4]{AnthonyBartlett1999}.
Thus, $\Nueral_m\subseteq\Leb_{p}(\Omega)$ and there exists a surjection $\Gamma_m$ from $\RR^m$ onto $\Nueral_m$.
It is easy to check that $\Nueral_m$ is weakly* closed if the domains of coefficients are bounded and closed.
Since $\left\{\Nueral_m:m\in\NN\right\}$ satisfies the universal approximation in $\Cont(\Omega)$, as exemplified in \cite{Cybenko1989,MhaskarMicchelli1994}, $\left\{\Nueral_m:m\in\NN\right\}$ also satisfies the universal approximation in
$\Leb_{p}(\Omega)$.

Moreover, $\langle f,\xi^{\ast}\rangle=\int_{\Omega}f(x)\xi^{\ast}(x)\ud x$ for $f\in\Leb_{p}(\Omega)$ and $\xi^{\ast}\in\Leb_q(\Omega)$. Thus, the input data for standard sigmoid neural networks or convolutional neural networks
can be equivalently transferred to the linear-functional data $\vxi_n^{\ast}\in\left(\Leb_q(\Omega)\right)^{N_n}$, such as mollifiers and convolutions.
For various classical data, we can verify that $\Funset_{\Data}$ is relatively compact, such as scientific sensor fusion related to kernels.
The multi-loss function $\loss_n$ is represented by the classical loss function $\closs$ as shown in Equation \eqref{eq:lossfun-exa}. Typically, $\closs$ is chosen as a local Lipschitz continuous loss function for developing algorithms for multi-layer neural networks.
Thus, Lemma \ref{Lm:loss} (b) ensures that $\Loss$ is uniformly local Lipschitz continuous.
This demonstrates that Condition (II) holds true.

For various problems of machine learning, the expected risks may be unidentified or uncertain, while the empirical risks are computable through two-layer neural networks based on $\vxi_n$ and $\loss_n$.
By employing numerical experiments, Condition (I) is considered to hold true, when the convergence of $\risk_n(f)$ exists
for any test function $f\in\Leb_{p}(\Omega)$.

Therefore, the pseudo-approximation theorems and convergence theorems hold true for various categories of two-layer neural networks.
In our forthcoming work, according to universal approximation in Sobolev spaces for sigmoid functions and reproducing kernels, we will investigate composite algorithms that incorporate artificial neural networks and support vector machines to adaptively solve partial differential equations by employing the theory of regularized learning.
}
\end{example}

Finally, we study a simple example of an ill-posed problem in the Euclidean space $\RR^2$ equipped with the 2-norm.
\begin{example}\label{exa:finite-dim}
{\rm
Since $\RR^2$ is a finite-dimensional Hilbert space, the weak* topology of $\RR^2$ is equal to the norm topology of $\RR^2$
and $\langle f,\xi^{\ast}\rangle=f\cdot \xi^{\ast}$ for $f,\xi^{\ast}\in\RR^2$.
Let
\[
\vA:=
\begin{pmatrix}
1&0\\
0&0\\
0&0
\end{pmatrix},
\quad
\vA_n:=
\begin{pmatrix}
1&0\\
0&1/n\\
0&0
\end{pmatrix}
\text{ for }n\in\NN,
\quad
\vb:=
\begin{pmatrix}
1\\1\\1
\end{pmatrix}.
\]
Now, we consider the linear equation $\vA\vc=\vb$.
Let an expected risk function $\risk(f):=\norm{\vA f-\vb}^2_2$ for $f\in\RR^2$.
Thus, $\Sset^0\left(\RR^2\right)$ is a collection of all least-squared solutions of $\vA\vc=\vb$
and the best-approximate solution of $\vA\vc=\vb$
is the minimum-norm element of $\Sset^0\left(\RR^2\right)$ as stated in \cite[Definition 2.1]{EnglHankeNeubauer2000}.

We have the specific data
\[
\vxi_n:=\left(\xi_{n1}^{\ast},\xi_{n2}^{\ast},\xi_{n3}^{\ast}\right),
\quad
\vy_n:=\left(1,1,1\right),
\quad\text{for all }n\in\NN.
\]
where
\[
\xi_{n1}^{\ast}:=
\begin{pmatrix}
1\\0
\end{pmatrix},
~
\xi_{n2}^{\ast}:=
\begin{pmatrix}
0\\1/n
\end{pmatrix},
\text{ and }
\xi_{n3}^{\ast}:=
\begin{pmatrix}
0\\0
\end{pmatrix}.
\]
It is easy to check that $\Funset_{\Data}$ is relatively compact in $\RR^2$.
Let a multi-loss function
\[
\loss_n\left(\vxi^{\ast},\vy,\vt\right):=\frac{1}{3}\sum_{k=1}^3\closs\left(\xi^{\ast}_{k},y_{k},t_k\right),
\quad\text{for }\vxi^{\ast}\in\RR^{3\times2},
~\vy\in\RR^3,~\vt\in\RR^3,
\]
where $n\in\NN$ and
\[
\closs\left(\xi^{\ast},y,t\right):=\left\{
\begin{array}{cl}
3\abs{t-y}^2,&\text{if }y\in\left\{1\right\},\\
0,&\text{otherwise},
\end{array}
\right.
\]
for $\xi^{\ast}\in\RR^2,~y\in\RR,~t\in\RR$.
Since $L$ is a local Lipschitz continuous loss function,
Lemma \ref{Lm:loss} (b) ensures that $\Loss$ is uniformly local Lipschitz continuous.
This demonstrates that Condition (II) holds true.

Moreover, $\risk_n(f)=\norm{\vA_nf-\vb}^2_2$ for $f\in\RR^2$. Thus, $\risk_n$ converges pointwise to $\risk$ when $n\to\infty$.
This demonstrates that Condition (I) holds true.
\cite[Theorem 2.5]{EnglHankeNeubauer2000} ensures that the best-approximate solution of $\vA\vc=\vb$ is
\[
f^0:=\vA^{\dagger}\vb=\begin{pmatrix}1\\0\end{pmatrix},
\]
and the best-approximate solution of $\vA_n\vc=\vb$ is
\[
f_n:=\vA_n^{\dagger}\vb=\begin{pmatrix}1\\n\end{pmatrix},
\]
where $\vA^{\dagger}$ and $\vA_n^{\dagger}$ are the pseudo inverses of $\vA$ and $\vA_n$, respectively.
Thus, $f_n$ is also a minimizer of $\risk_n$ over $\RR^2$.
But $f_n$ is not convergent when $n\to\infty$.
This demonstrates that $f_n$ is not a well-posed approximate solution of $\vA\vc=b$ for any $n\in\NN$.

Let $\regfun(r):=r^2$ for $r\in[0,\infty)$.
Thus, \cite[Theorem 5.1]{EnglHankeNeubauer2000} ensures that there exists the unique approximate solution $f_n^{\lambda}\in\Sset_{n}^{\lambda}\left(\RR^2\right)$ which is written as
\[
f_n^{\lambda}=\left(\vA_n^T\vA_n+\lambda\vI\right)^{-1}\vA_n^T\vb
=\begin{pmatrix}1/(1+\lambda)\\n/(1+\lambda n^2)\end{pmatrix}.
\]
where $\vI$ is an identity matrix.
According to Theorem \ref{Thm:ConvThm}, $f_n^{\lambda}$ can become a well-posed approximate solution for a $n\in\NN$ and a $\lambda>0$.
For example, if we choose $\lambda_n:=1/\sqrt{n}$, then $f_n^{\lambda_n}\to f^0$ when $n\to\infty$.
However, if we choose $\lambda_n:=1/n^2$, then $f_n^{\lambda_n}$ is not convergent when $n\to\infty$.
This shows that the approximate solution may not be convergent or stable for any pair $n,\lambda$ even when $n\to\infty$ and $\lambda\to0$.
}
\end{example}
\begin{remark}
In Example \ref{exa:finite-dim}, it is evident that the empirical risk functions are not equicontinuous globally, while
they are equicontinuous on all bounded subset.
Even though we may not solve ill-posed problems directly through empirical risks, regularized learning provides an alternative method for obtaining well-posed approximate solutions.
Example \ref{exa:finite-dim} further illustrates that the selection of $\lambda$ impacts the convergence of the approximate solutions for ill-posed problems.
Similar to Theorem \ref{Thm:ConvThm-lambda}, we will attempt to investigate the open problem of selecting the adaptive $\lambda$ for specific machine learning problems through the utilization of the weak* topology.
\end{remark}

\section{Final Remark}\label{sec:final}

In this article,
the theory of Banach spaces serves as fundamental basis for regularized learning in the analysis of linear-functional data.
Through the techniques of weak* topology, we finalize the proofs of the representer theorems, pseudo-approximation theorems, and convergence theorems in regularized learning.
My philosophical concept is to approximate unknown rules through explicit models and observed data using the theory of regularized learning.
The objective of regularized learning is to identify an approximate solution that can be effectively computed by a machine.
The field of regularized learning provides an additional pathway for investigating computational learning theory including:
\begin{itemize}
\item[$\bullet$] interpretability through approximation theory,
\item[$\bullet$] nonconvexity and nonsmoothness through optimization theory,
\item[$\bullet$] generalization and overfitting through regularization theory.
\end{itemize}
In my opinion, the exact solutions of the original problems are globally approximated by machine learning algorithms that provide local interpretability through linear-functional data.
According to the theorems of regularized learning,
the existence and convergence of the approximate solutions can be ensured by utilizing nonconvex and nonsmooth loss functions.
Hence, the techniques of nonconvex and nonsmooth optimization will be employed to develop the iterative algorithms.
Since weak* compactness is essential in the proof of the weak* convergence of the approximate solutions, the regularization terms play a crucial role in verifying generalization and preventing overfitting.
Based on the characteristics of weak* convergence in the convergence theorems, we will aim to introduce locally convergent rates and error bounds for the approximate solutions at certain observed inputs with specific structures, such as relative compactness.
In our forthcoming paper, we intend to investigate novel mathematical techniques and concepts to study the theory of regularized learning, such as sparse machine learning in $l_{1/2}$ spaces using generalized dual topology.

Furthermore, the depictions of regularized learning exhibit similarities to the equations used in inverse problems.
By employing classical approaches to inverse problems,
convergent analysis typically necessitates a condition of uniform convergence.
However, verifying uniform convergence directly poses a significant challenge in numerous machine learning problems because the original problem is frequently unidentified or uncertain.
Indeed, pointwise convergence is considered to be a less stringent condition compared to uniform convergence.
Moreover, the pointwise convergence of empirical risks can be verified through test functions in practical applications or postulated through the representation of natural models.
Hence, in the convergence theorems, we focus on the condition of pointwise convergence, exemplified by Condition (I).
Unfortunately, the condition of pointwise convergence
is not sufficient to prove the convergence of the approximate solutions to the exact solutions.
Therefore, there arises a fundamental issue whether a robust convergence condition can be substituted by a weaker convergence condition together with an additional checkable condition.
Our proposition posits that the weaker condition and the additional condition can be independently checked.
Since linear-functional data is always known, the additional condition of linear-functional data is feasible to check as shown in Lemma \ref{Lm:weak-equicont}.
By the construction of linear-functional data, relative compactness emerges as a natural additional condition of linear-functional data as explained in Remark \ref{Rm:Equicontin-data}.
The additional condition of loss functions is straightforward due to their typically simplistic forms.
Clearly, the additional condition of linear-functional data and loss functions is independent on the original problem, exemplified by Condition (II).
The additional condition also implies
the specific weak* equicontinuity of empirical risks as shown in Proposition \ref{Pro:EmpRisk} (c).
As stated in Section \ref{sec:ProofConvThm}, the weaker condition and the additional condition are sufficient to prove the convergence theorems.

\begin{wrapfigure}{r}{0.32\textwidth}
  \vspace{-20pt}
  \begin{center}
   \includegraphics[width=0.28\textwidth]{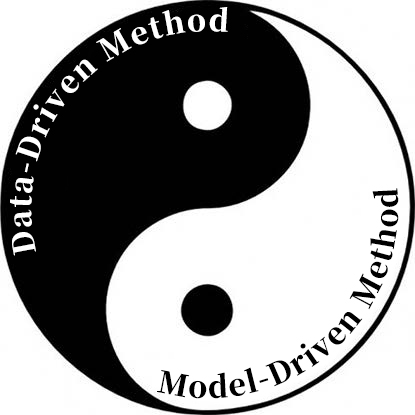}
  \end{center}
  \vspace{-40pt}
\end{wrapfigure}
Specifically, data-driven methods are commonly employed for the analysis of black-box models
and model-driven methods are commonly employed for the analysis of white-box models.
In my opinion, both black-box and white-box models have the capacity to approximate the real world.
Additionally, both data-driven and model-driven methods can provide approximate solutions for developing construct decision rules.
For a particular illustration in Section~\ref{sec:AppCompAlg}, the regularized learning introduces a novel approach for developing composite algorithms to integrate
black-box and white-box models into a unified system.
By iteratively combining approximate solutions from data-driven and model-driven methods, the theorems in regularized learning guarantee the construction of well-formed composite approximate solutions.
Our initial concept draws inspiration from Eastern philosophies, specifically the golden mean and the Tai Chi diagram.
In our present investigations, we are delving into the realm of big data analysis within the fields of education and medicine through the composite algorithms that combine black-box and white-box models, such as support vector machines, artificial neural networks, and decision trees.

{\bf Comments:} This paper is the corrected version of the publication in the Journal of the Operations Research Society of China (DOI: https://doi.org/10.1007/s40305-025-00596-x).

\bmhead{Acknowledgments}

The author would like to express his gratitude to Professor Liren Huang at the South China Normal University.
Huang was the master advisor of Ye at the South China Normal University.
When Ye was back to work at Guangzhou, Huang had already retired for a long time.
Huang still gave a great help of this article.
The author is supported by
the National Natural Science Foundation of China \#12071157 and
the Guangdong Basic and Applied Basic Research Foundation \#2024A1515012288.
The author would also like to thank for the suggestions of corrections from the GPT-6 (Astra).

%



%
\end{document}